\documentclass[journal]{IEEEtran}
\usepackage{mathtools}
\usepackage{balance}
\usepackage{amsmath,amsfonts,amssymb,amsthm}
\newtheorem{lemma}{Lemma}
\newtheorem{theorem}{Theorem}
\newtheorem{corollary}{Corollary}
\newtheorem{assumption}{Assumption}
\usepackage{algorithm}  
\usepackage{algorithmicx} 
\usepackage{algpseudocode}
\usepackage{array}
\usepackage{ulem}
\usepackage{wrapfig}
\usepackage{textcomp}
\usepackage{stfloats}
\usepackage{url}
\usepackage{booktabs}
\usepackage{tabularx}
\usepackage{makecell}
\usepackage{multirow}
\usepackage{threeparttable}  
\usepackage{verbatim}
\usepackage{cite}
\usepackage{graphicx}  
\usepackage{float}  
\usepackage{soul}
\usepackage[caption=false,font=normalsize]{subfig}
\usepackage{color, xcolor}
\usepackage[protrusion=true,expansion=true]{microtype}
\microtypecontext{spacing=nonfrench}
\soulregister{\cite}7
\soulregister{\citep}7 
\soulregister{\citet}7 
\soulregister{\ref}7 
\soulregister{\pageref}7
\newcommand{\C}{\mathcal{C}}
\newcommand{\Ci}{\mathcal{C}_i}
\begin{document}

\title{FedAPA: Federated Learning with Adaptive Prototype Aggregation Toward Heterogeneous Wi-Fi CSI-based Crowd Counting}

\author{Jingtao Guo,~\IEEEmembership{Student Member,~IEEE}, Yuyi Mao,~\IEEEmembership{Senior Member,~IEEE}, \\ and Ivan Wang-Hei Ho,~\IEEEmembership{Senior Member,~IEEE}
\thanks{}
\thanks{}}

\markboth{Journal}%
{Shell \MakeLowercase{\textit{et al.}}: A Sample Article Using IEEEtran.cls for IEEE Journals}


\maketitle

\begin{abstract}
Wi-Fi channel state information (CSI)-based sensing provides a non-invasive, device-free approach for tasks such as human activity recognition and crowd counting, but large-scale deployment is hindered by the need for extensive site-specific training data. Federated learning (FL) offers a way to avoid raw data sharing but is challenged by heterogeneous sensing data and device resources. This paper proposes FedAPA, a collaborative Wi-Fi CSI-based sensing algorithm that uses adaptive prototype aggregation (APA) strategy to assign similarity-based weights to peer prototypes, enabling adaptive client contributions and yielding a personalized global prototype for each client instead of a fixed-weight aggregation. During local training, we adopt a hybrid objective that combines classification learning with representation contrastive learning to align local and global knowledge. We provide a convergence analysis of FedAPA and evaluate it in a real-world distributed Wi-Fi crowd counting scenario with six environments and up to 20 people. The results show that our method outperform multiple baselines in terms of accuracy, F1 score, mean absolute error (MAE), and communication overhead, with FedAPA achieving at least a 9.65\% increase in accuracy, a 9\% gain in F1 score, a 0.29 reduction in MAE, and a 95.94\% reduction in communication overhead.
\end{abstract}


\section{Introduction}
\IEEEPARstart{W}{i-Fi} sensing technologies have seen widespread use
in various fields, including indoor positioning~\cite{guo2023fedpos}
and crowd counting~\cite{guo2025rssi}. Among these applications, crowd counting is crucial for optimizing space allocation in restaurants, transportation hubs, and public facilities, and for improving energy efficiency through automatic control of air conditioning and electrical systems. Current crowd counting methods fall into two categories: device-based and device-free approaches~\cite{pcsurvey}. Device-based methods require each individual in the crowd to carry a sensor~\cite{weppner2013bluetooth}, which is often inconvenient or infeasible. Device-free methods are further divided into sensor-based, image-based, and radio frequency (RF)-based counting. Sensor-based methods (e.g., $\text{CO}_2$, acoustic, LiDAR, infrared) require additional infrastructure, increasing deployment costs. Image-based techniques, although widely used~\cite{gao2020cnn}, raise privacy concerns and are sensitive to non-line-of-sight (NLOS) and lighting conditions, which limits their applicability in many indoor and outdoor scenarios~\cite{xiao2022survey}. RF-based systems, using signals such as Bluetooth and Wi-Fi, offer a contactless and NLOS-capable alternative, with Wi-Fi-based methods being particularly attractive due to their low cost and widespread availability~\cite{wifivision}.

Wi-Fi signal scattering and reflection vary with the number of people in an indoor space, forming useful patterns for crowd counting and occupancy detection. The two main radio signal features used for human activity sensing are the received signal strength indicator (RSSI) and channel state information (CSI) matrix~\cite{pcsurvey, xiao2022survey}. While RSSI-based crowd counting techniques have shown potential \cite{depatla2015occupancy}, they fail to meet the robustness and scalability requirements for accurate real-world occupancy measurement due to distortion caused by multi-path fading~\cite{pcsurvey},~\cite{yang2013rssi}. In contrast, the CSI matrix provides amplitude and phase information across multiple sub-carriers, is able to capture subtle environmental changes, making it a great choice for crowd counting \cite{liu2019wireless, wifivision, wang2024wall}.

Several CSI-based occupancy measurement systems \cite{zhao2019deepcount, jiang2023pa, guo2022csi, guo2025rssi, wang2025guiding} use support vector machines (SVM), long short-term memory (LSTM) networks, and convolutional neural networks (CNN) to estimate the number of people in a target area. However, these models are typically trained in isolation for each environment, which limits scalability in large deployments. A natural idea is to aggregate CSI data from many devices for centralized cloud training, but sharing Wi-Fi sensing data across environments and ownership domains introduces serious privacy and security risks \cite{tan2022commodity}. For example, attackers can generate adversarial CSI samples using jamming signals based on normal CSI data, causing incorrect or even malicious model outputs \cite{liu2023time}. Moreover, edge devices usually have limited computation due to low-cost and low-power constraints. Offloading raw CSI to the cloud for centralized processing is also problematic: the high dimensionality and sampling rate of CSI create a continuous data stream that can saturate communication bandwidth and interfere with basic WiFi services (e.g., Internet access). Therefore, an efficient distributed algorithm is needed to reduce communication overhead while preserving sensing accuracy.

Federated learning (FL) allows multiple edge devices or organizations to collaboratively train a global model without sharing raw data, coordinated by a central server \cite{yang2019federated}. However, deploying FL in distributed Wi-Fi sensing systems is challenged by statistical and model heterogeneity. CSI data from different environments are typically non-independent and identically distributed (non-IID), and edge devices have varying computational resources, making it difficult to learn a single global model \cite{pflsurvey, heterflsurvey}. For example, meeting rooms and offices exhibit different propagation and occupancy patterns, while hardware ranges from laptops with Intel 5300 NICs to resource-constrained ESP-32 modules, which call for different model architectures. Consequently, naive strategies such as simple model averaging are often insufficient to obtain an effective global model across such heterogeneous environments and devices.

To address these challenges, researchers have proposed various modifications to local training and global aggregation within the FedAvg framework \cite{fedavg}. For example, MOON \cite{moon} incorporated global model knowledge into local training to reduce the drift of local updates, guiding the update of the global model toward an unbiased trajectory. Another common strategy involved sharing prototypes (the mean of feature embeddings for each label) to regularize local training \cite{fedproto}. However, these existing methods are often limited by either focusing on addressing a single type of heterogeneity or achieving only marginal performance improvements, making them less effective for clients with diverse data and resource distributions, particularly in scenarios where some clients possess data from only a few classes \cite{li2022federated}.

This paper aims to address the challenges of statistical and model heterogeneity, including feature-skew, label-skew, and diverse model architectures, which commonly arise in distributed Wi-Fi people-counting applications. Specifically, we introduce a novel FL algorithm, FedAPA, which promotes collaboration across clients with similar data distributions by adaptive prototype aggregation (APA) module. The main contributions of this paper are five-fold:
\begin{itemize}
\item We present \textbf{FedAPA}, a FL algorithm for distributed Wi-Fi CSI crowd counting that handles both statistical heterogeneity (non-IID data) and model heterogeneity. FedAPA coordinates clients through class prototypes while keeping per-client personalization.

\item We design an APA module that weights and aggregates client prototypes by similarity. This enables knowledge sharing among similar clients and reduce communication overhead by exchanging compact prototypes instead of gradients or full models.

\item We introduce a hybrid local objective that aligns local and global knowledge by combining cross entropy with a prototype contrastive loss (PCL). A warm-up schedule gradually increases the PCL weight to avoid early instability and reduce drift between local and global representations.

\item We present a convergence analysis of the proposed algorithm and derive a nonconvex convergence rate, highlighting how prototype similarity-aware aggregation and warm-up strategies affect convergence.

\item We validate FedAPA on a real-world Wi-Fi CSI crowd-counting dataset across multiple architectures and non-IID splits. FedAPA consistently outperforms several FL-based Wi-Fi sensing baselines while lowering client-side communication overhead, showing practical gains for resource-constrained deployments.
\end{itemize}

\begin{table}[t]
\centering
\footnotesize
\setlength{\tabcolsep}{4pt}
\renewcommand{\arraystretch}{1.05}
\caption{Key notations used in this article.}
\begin{tabular}{@{}l >{\raggedright\arraybackslash}p{0.64\linewidth}@{}}
\toprule
\textbf{Symbol} & \textbf{Meaning} \\
\midrule
$N$ & Number of clients (Wi-Fi devices). \\
$\mathcal C,\ \mathcal C_i$ & Global label set; labels present at client $i$. \\
$\mathcal D_i,\ \mathcal D_i^{c}$ & Dataset at client $i$; class-$c$ subset at client $i$. \\
$w_i=(w_i^\theta,w_i^h)$ & Encoder and classifier parameters of client $i$. \\
$\mathbf r_i(\mathbf h)$ & Embedding of CSI input $\mathbf h$ at client $i$. \\
$\mathbf p_i^{c},\ \mathbf P_i$ & Class-$c$ prototype; prototype set of client $i$. \\
$\mathbf P$ & Stacked prototype sets of all clients. \\
$\mathbf q_i^{c},\ \mathbf Q_i$ & Personalized class-$c$ prototype; personalized set for client $i$. \\
$\tau$ & Temperature for aggregation and contrastive logits. \\
$\lambda$ & Weight of prototype-based losses. \\
$S,\ T,\ K{=}TS$ & Local SGD steps per round; rounds; total local steps. \\
$\eta$ & Learning rate for local updates. \\
$\mathcal L_{ce},\ \mathcal L_g,\ \mathcal L_c,\ \Phi$ & Cross-entropy loss; contrastive loss w.r.t.\ $\mathbf Q_i$; contrastive loss w.r.t.\ $\mathbf P$; combined prototype regularizer. \\
$L_{\max},\ \sigma^2,\ G$ & Smoothness bound after warm-up; gradient-variance bound; bound on $\mathbb E\|g_{i,t,s}\|^2$. \\
$G_{i,t}$ & Sum of gradient norms across $S$ local steps at round $t$ for client $i$. \\
$k_{\mathrm{glob}}$ & Prototype-movement constant in the convergence bounds. \\
$\Gamma(\tau)$ & Sensitivity term in prototype-refresh error, $\Gamma(\tau)=c_\Phi(\tau)(1+L_{\mathrm{agg}}(\tau))$. \\
$L_{\inf},\ \Delta_i$ & Lower bound on local objectives; initial gap for client $i$ after warm-up. \\
\bottomrule
\end{tabular}
\end{table}

The remainder of this paper is structured as follows: Section \uppercase\expandafter{\romannumeral2} reviews related works, providing the literature context for our proposed method. Section \uppercase\expandafter{\romannumeral3} introduces CSI background and outlines the problem statement. Section \uppercase\expandafter{\romannumeral4} presents our proposed federated learning algorithm. Section \uppercase\expandafter{\romannumeral5} provides a convergence analysis of our proposed approach. Section \uppercase\expandafter{\romannumeral6} describes the experimental setup and results. Section \uppercase\expandafter{\romannumeral7} and \uppercase\expandafter{\romannumeral8} conclude with a summary of findings and contributions.

\section{Related Works}

\subsection{Wi-Fi CSI-Based crowd Counting System}

Recent advancements in Wi-Fi CSI-based crowd counting systems~\cite{zou2017freecount, hou2022dasecount} have aimed to address the challenge of data distribution shifts when pretrained models are deployed in new environments. However, many models still require additional fine-tuning to achieve optimal performance in new environments. With the emergence of 6G technology, it is now possible to aggregate large amounts of data from diverse environments, enabling the training of a Wi-Fi sensing foundation model.

\subsection{Collaborative Wi-Fi-based Sensing System}
Federated learning (FL) trains Wi-Fi sensing models across edge devices without sharing raw CSI by using on-device updates and server-side aggregation. In Wi-Fi sensing, FL has been applied to handle non-IID data: CARING~\cite{caring} reweights client updates by local performance to improve cross-domain recognition, and WiFederated~\cite{wifederated} scales training across multiple locations. However, these methods often need extra training rounds or yield limited gains under strong non-IID conditions. To reduce labeling costs, Zhang et al.~\cite{zhang2022cross} propose a cross-domain FL framework that synthesizes wireless-like data from images and uses MMD-based domain adaptation, adding transfer and adaptation overhead. Liu et al.~\cite{liu2022vertical} introduce a vertical Federated Edge Learning (FEEL) framework that exchanges low-dimensional intermediate features for collaborative object and motion recognition while protecting raw data, but it mainly addresses feature-partitioned settings rather than label skew. In contrast, our approach targets label skew, feature skew, and resource limits jointly, and reduces communication by sharing and aggregating class prototypes instead of full models.

\subsection{Personalized FL}
To handle statistical heterogeneity in data and models, personalized FL (pFL) methods aim to learn client-specific models rather than a single global one. Instead of only optimizing a centralized model via distributed training, pFL explicitly tailors each client's model to its local data and, when needed, architecture. These methods are commonly grouped into three categories:
\begin{enumerate}
    \item \textbf{Global model with local fine-tuning}, as in WiFederated~\cite{wifederated}, which first trains a shared model and then adapts it on each client.
    \item \textbf{Personalized layers or representations}, such as FedRep~\cite{fedrep}, which shares a common encoder and trains local heads; FedBN~\cite{fedbn}, which keeps batch normalization layers local to mitigate feature shift; and FedProto~\cite{fedproto}, which exchanges class prototypes instead of model parameters to support heterogeneous architectures.
    \item \textbf{Personalized aggregation}, which builds local models from client-dependent combinations of global and local parameters, for example FedALA~\cite{fedala} with element-wise aggregation and FedAMP~\cite{fedamp} with similarity-weighted aggregation of client models.
\end{enumerate}
FedAPA belongs to class (2): it communicates and aggregates class prototypes through an APA module, so clients share compact class summaries while keeping their local models and statistics private.

\section{Preliminaries and Problem Statement}
In this section, we first briefly introduce Wi-Fi CSI, it is then followed by the formulation of a general prototype-based FL framework for Wi-Fi CSI-based crowd counting.

\subsection{Wi-Fi CSI}
Wi-Fi technologies utilize orthogonal frequency division multiplexing (OFDM), where data and reference symbols are transmitted via $S_{sub}$ orthogonal subcarriers. Wi-Fi CSI represents the estimated channel response between a Wi-Fi transmitter (Tx) and a receiver (Rx) for each subcarrier. This information is readily available at the receiver and can be extracted using CSI extraction middlewares such as Nexmon CSI extractor~\cite{gringoli2019free} Assuming $K_{pkt}$ data packets are transmitted by a Wi-Fi transmitter with $N_t$ antennas to a receiver with $N_r$ antennas, each Wi-Fi CSI sample is represented as a tensor with dimensions $N_t \times N_r \times K_{pkt} \times S_{sub}$. \begin{figure}[htbp]
        \centering
        \subfloat[Living room scenario]{\includegraphics[width=0.48\columnwidth]{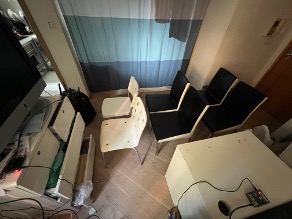}\label{fig:sub1}}\hfill
        \subfloat[Conference room scenario]{\includegraphics[width=0.48\columnwidth]{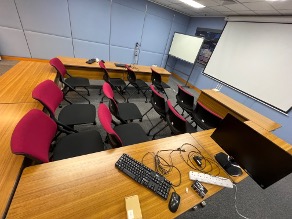}\label{fig:sub2}}
        \caption{Illustration of two different environment layout for Wi-Fi CSI-based crowd counting data collection.}
        \label{fig:pic1}
\end{figure}The channel coefficient between the $n_t$-th transmit antenna and $n_r$-th receive antenna for the $k$-th packet at subcarrier $s$ is expressed as:
\begin{equation}
\label{eqn_1}
{h}_{k,s}^{n_t,n_r} = |{h}_{k,s}^{n_t,n_r}|e^{j \cdot \theta_{k,s}^{n_t,n_r}},
\end{equation}
where $|{h}_{k,s}^{n_t,n_r}|$ and $\theta_{k,s}^{n_t,n_r}$ represent the amplitude and phase coefficients, respectively. The label of each Wi-Fi CSI sample, $y$, indicates the number of people present.

Consider two scenarios, each with a Wi-Fi transmitter-receiver pair, as shown in Fig. \ref{fig:pic1} (a) and (b). Each environment, such as a conference room versus a living room, has unique properties that cause variations in signal characteristics. The challenge lies in training models that generalize well across different environments without extensive data collection. A model trained in one environment may not perform well in another, necessitating the design of a collaborative training framework that shares knowledge across environments to improve generalization and reduce the need for data collection.

\begin{figure*}[!tb]
    \centering
    \includegraphics[width=0.7\textwidth]{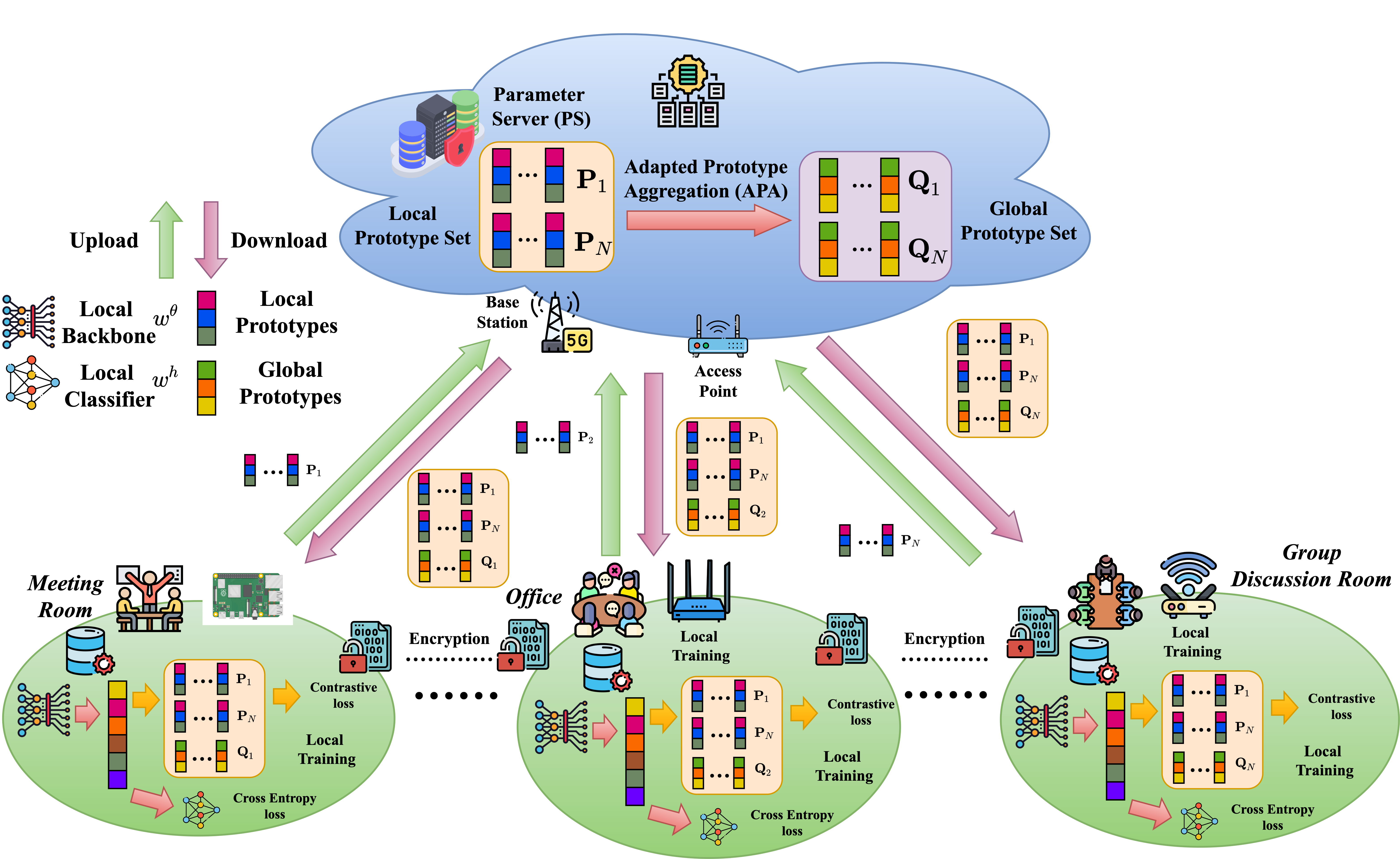}
    \caption{Overview of the proposed framework. \textbf{Bottom:} Clients with varying computational resources capture CSI data in distinct environments (e.g., Meeting Room, Office) and extract local prototypes from CSI data. \textbf{Top:} The Parameter Server aggregates uploads via Adapted Prototype Aggregation (APA) to form personalized sets $\mathbf{Q}$. \textbf{Training:} Clients optimize backbones $w^\theta$ and classifiers $w^h$ using the returned global $\mathbf{P}$ and personalized $\mathbf{Q}$ sets via a hybrid loss.}
    \label{fig:pic2}
\end{figure*}

\subsection{Prototype-based FL for Wi-Fi CSI-based crowd Counting}

FL leverages Wi-Fi CSI data from multiple environments to learn crowd counting models. Standard FL methods, such as FedAvg~\cite{fedavg}, train a single shared model across $N$ clients (Wi-Fi receivers in different environments). In practice, Wi-Fi CSI-based crowd counting is highly non-IID: people density and movement patterns vary across spaces (for example, a busy shopping mall on weekends versus an office after hours), and receivers also differ in computational resources. As a result, a single global model is often suboptimal, and personalized models tailored to each environment are needed.

We conceptualize a client model $w_i$ as the combination of a encoder $w^{\theta}_i$ and a classifier $w^h_i$. The embedding of the $l$-th Wi-Fi CSI matrix with label $y$ at client $i$, denoted $\mathbf{h}_{i}^{l,y}$, is computed as $\mathbf{r}(\mathbf{h}_{i}^{l,y}) \triangleq f_i\left(\mathbf{h}_{i}^{l,y}; w^{\theta}_i\right)$, and the prediction is given by $g_i\left(\mathbf{r}(\mathbf{h}_{i}^{l,y}); w^{h}_i\right)$. Each client will then generate its prototype for each class $c$ as $\mathbf{p}_{i}^{c} \triangleq \frac{1}{|\mathcal{D}_{i,c}|}\sum_{l=1, y \in c}^{|\mathcal{D}_{i,c}|} \mathbf{r}_{i}^{l, y}$ at the $t$-th communication round, where $\mathcal{D}_{i, c} = \left\{(\mathbf{h}, y) \in \mathcal{D}_i : y=c\right\}$, and $|\mathcal{D}_{i,c}|$ is the number of samples of class $c$ for client $i$. The client collaboration involves sharing their prototypes with a central server and generating global prototypes $\mathbb{P}$ for each client, which brings external information to other clients. The original prototype-based FL algorithm for Wi-Fi CSI-based crowd counting is formulated as:
\begin{equation}
\label{eqn_2}
\mathop{\min}_{w_1, w_2 \cdots, w_N} \frac{1}{N} \sum_{i=1}^{N} \mathbb{E}_{(\mathbf{h}, y) \in \mathcal D_i}\Bigl[\mathcal{L}\left(\mathbf{h}, y, \mathbb{P}; w_i\right)\Bigr],
\end{equation}
with the local loss function $\mathcal{L}$ defined as 
\begin{equation}
\label{eqn_3}
    \begin{aligned}
    \mathcal{L}\left(\mathbf{h}, y, \mathbb{P}; w_i\right) &= \mathcal{L}_{ce}\left(g_i\left(f_i\left(\mathbf{h}; w_i^{\theta}\right); w_i^h\right), y\right) \\
    &\quad + \lambda \, \mathcal{L}_{g}\left(f_i\left(\mathbf{h}; w_i^{\theta}\right), \mathbb{P}\right),
    \end{aligned}
\end{equation}
where $\mathcal{L}_{ce}(\cdot, \cdot)$ is the cross-entropy loss for classification, $\mathcal{L}_{g}(\cdot, \cdot)$ is a representation learning loss that aligns local embeddings with global prototypes, and $\lambda$ is a hyperparameter balancing the two loss components. Here,
$N$ is the total number of clients, $\mathcal{D}_{i}$ is the local dataset of client $i$, and $\mathbb{P}$ is the global prototypes that average clients' uploaded prototype set. In our proposed algorithm, we replaces $\mathbb P$ by personalized global prototypes $\mathbf Q_i$ built by similarity-aware aggregation, and adds an inter-client loss term that levergae the whole uploaded prototypes $\mathbf{P} = \left\{\mathbf{P}_i\right\}_{i=1}^{N}$. Details of $\mathbf{Q}_i$ generation and the local training procedure are described in Section~\ref{sec:fedapa}.

\section{FedAPA}
\label{sec:fedapa}
This section presents the FedAPA framework, illustrated in Fig.~\ref{fig:pic2}. We consider multiple crowd counting environments connected to a parameter server (PS). FedAPA introduces a prototype-based aggregation scheme inspired by prototype learning~\cite{prototypicalfs}, which fuses feature representations from heterogeneous data by aggregating class-wise embeddings. Unlike prior prototype-based FL methods~\cite{fedpcl, fedproc, fedproto} that form a single global prototype via uniform averaging, FedAPA constructs personalized global prototypes based on inter-client similarity. Each client receives its personalized prototypes and the full prototype set from the PS, and trains a local model with a hybrid loss that combines classification and prototype-wise contrastive learning. This aligns global and local representations and strengthens inter-client knowledge sharing in the latent space. The following subsections detail the three components of FedAPA: \textbf{Prototypes as compact information carriers}, \textbf{Server aggregation}, and \textbf{Local training}.

\subsection{Prototypes as compact information carriers}
\label{subsec:proto}
In Wi-Fi sensing task, raw CSI amplitude matrix is usually high-dimensional, making direct sharing across clients impractical due to communication constraints. To enhance client-specific knowledge sharing while reducing the communication overhead, we leverage a prototype-based information-sharing strategy. By Exchanging prototypes per client, the communication cost is far smaller than pushing full models or sharing raw CSI streams. Because prototypes are computed from the information of last round, they carry fresh information forward without exposing raw CSI. At $t$-th round, each client $i$ updates its local model $w_{i,t}$ on $\mathcal D_i$ and maps each CSI matrix $\mathbf{h}$ to an embedding $\mathbf{r}_i(\mathbf{h}_{i}^{l,y})$. For every class $c$, the client forms a prototype as its mean embedding: 
\begin{equation}
\label{eqn_4}
\mathbf p_{i,t}^c := \frac{1}{m_i^c}\sum_{\mathbf h\in\mathcal D_i^c} \mathbf r_{w_{i,t,S}^\theta}(\mathbf h), \qquad m_i^c := |\mathcal D_i^c| \ge 1.
\end{equation}
Once its local prototype set $\mathbf{P}_i$ is constructed, the client uploads it to the PS for aggregation, facilitating client-specific compact information sharing derived from local datasets. 

\subsection{Server aggregation}
\label{subsec:agg}
After receiving the prototype set $\mathbf{P}$, the PS first computes pairwise inter-client cosine similarity\cite{fedamp, nguyen2010cosine} using prototypes of each class as follows:
\begin{equation}
    \label{eqn_5}
    \displaystyle s_{ij}^c=\frac{(\mathbf p_i^c)^{\top}\mathbf p_j^c}{\|\mathbf p_i^c\|_2\|\mathbf p_j^c\|_2},\,c \in \left(\mathcal C_i \cap \mathcal C_j\right),
\end{equation}
where $\mathbf{p}_{i}^c$ and $\mathbf{p}_{j}^c$ are the prototypes of class $c$ of the $i$-th and $j$-th clients, respectively. Let $\mathcal J_c(i):=\{j:\,c \in \left(\mathcal C_i \cap \mathcal C_j\right)\}$ be the eligible client set for client $i$ with class $c$. It then generates an adaptive weight $\alpha_{ij}^c$ of each class with softmax normalization:
\begin{equation} 
    \label{eqn_6}
    \alpha_{ij}^c = \frac{\exp \left(s_{ij}^c / \tau\right)}{\sum_{j \in \mathcal J_c(i)} \exp \left(s_{ij}^c / \tau\right)},
\end{equation}
where $\tau$ is the temperature hyperparameter used to sharpen the distribution.
Instead of fixed or uniform weights, the adaptive weights $\alpha_{ij}^c$ allow each client aggregates class-wise prototypes mainly from similar peers. This adaptive weighting strategy reduces negative transfer from heterogeneous clients~\cite{fedamp, ye2023personalized} and produces client-specific prototype sets instead of a single global average. In our convergence analysis (Section~\ref{sec:convergence}), the effect of similarity-weighted aggregation is captured by a Lipschitz constant $\Gamma(\tau)$ that depends on the softmax temperature $\tau$, and enters the bounds through the cross-client coupling term $\lambda\,\Gamma(\tau)\,k_{\mathrm{glob}}\,G$. With the normalized weights, the server aggregates the prototypes of class $c$ for client $i$ as
\begin{equation}
    \label{eqn_7}
    \mathbf{q}_{i}^c = \sum_{j \in \mathcal J_c(i)} \alpha_{ij}^c \, \mathbf{p}_{j}^c.
\end{equation}
After aggregation for all clients, the server returns the global personalized prototypes $\mathbf{Q}_i:= \left\{\mathbf{q}_{i}^c\right\}_{c \in |\mathcal{C}_i|}$ to client $i$. To handle clients with missing classes and stabilize prototype-based collaboration under label-skew non-IID distributions, we adopt a prototype padding mechanism. For any class $\hat{c}$ absent at client $i$, we construct a pseudo-prototype for $\mathbf{Q}_i$ and $\mathbf{P}_i$ by sample-weighted averaging of the corresponding prototypes from other clients, i.e.,
\begin{equation}
\label{eq:padding}
\mathbf{p}_{i}^{\hat{c}} = \mathbf{q}_{i}^{\hat{c}} = \frac{1}{|\mathcal{J}_{\hat{c}}|} \sum_{j \in \mathcal{J}_{\hat{c}}} \mathbf{p}_{j}^{\hat{c}},
\end{equation} 
where $\mathcal J_{\hat{c}} := \{j:\, \hat{c} \in \mathcal{C}_j, \hat{c} \notin \mathcal{C}_i\}$, and include this padded prototype in our contrastive representation learning. This yields a complete, globally consistent prototype set for each client, while simultaneously providing informative negative anchors in the contrastive loss for classes that are not observed locally, thereby improving cross-client feature alignment without sharing raw data. This approach promotes stronger collaboration among clients with similar data distributions while reducing computational costs compared to methods like FedAMP~\cite{fedamp}, which rely on model parameters for collaboration. Unlike approaches such as \cite{fedproto, fedproc, fedpcl} that average uploaded prototypes directly, our method iteratively assigns weights to clients with similar prototypes during aggregation. This process creates a positive feedback loop that strengthens collaboration among similar clients. Furthermore, it dynamically forms cohorts of similar clients, enhancing the overall effectiveness of collaboration. Apart from the global personalized prototypes, each client will also receive the full uploaded prototype set $\mathbf{P}$.

\subsection{Local training}
\label{subsec:local}
Upon receiving aggregated prototype $\mathbf{Q}_i$ and client prototypes $\mathbf{P}_{\mathcal{J}_i}$ from the PS, the primary objective of local training is to effectively mitigate the divergence caused by local updates while extracting both class-relevant and inter-client representational knowledge. To achieve this, we propose a hybrid learning architecture that utilizes a warm-up coefficient to integrate classifier learning with representation contrastive learning. Figure \ref{fig:pic3} illustrates the local training process for each client. \begin{figure*}[!t]
    \centering
    \includegraphics[width=\textwidth]{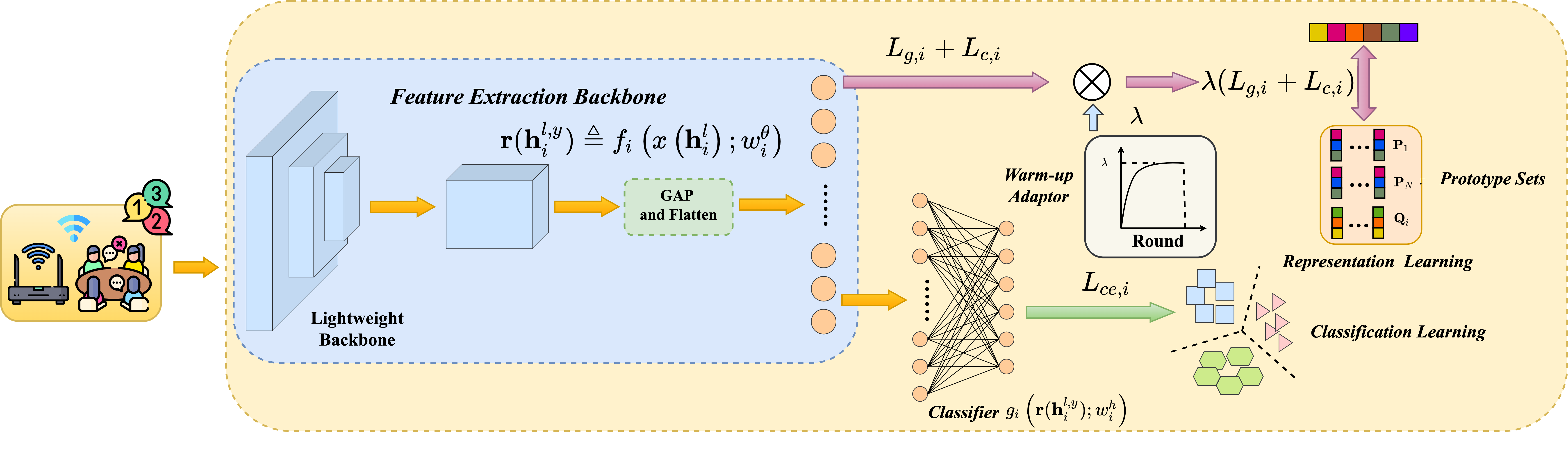}
    \caption{The encoder $w_i^{\theta}$ extracts features $\mathbf{r}_i$ from CSI input $\mathbf{h}$ for classification by $w_i^h$. The training objective integrates cross-entropy with contrastive alignment against prototype sets $\mathbf{Q}_i$ and $\mathbf{P}$, modulated by a warm-up coefficient $\lambda$ that progressively emphasizes representation learning after classification stabilizes.}
    \label{fig:pic3}
\end{figure*}To encourage $\mathbf{r}_i(\mathbf{h})$ extracted by local encoder to better align with its corresponding aggregated prototype so that it can learn more class-relevant and client-irrelevant knowledge, we introduce a global personalized prototype-based contrastive loss term, $\mathcal{L}_{g}$, as defined below: \begin{equation}
    \label{eqn_8}
    \mathcal{L}_{g} = \mathbb{E}_{(\mathbf{h}, y) \in \mathcal{D}_i} \Bigg[ -\log \frac{
        \exp \Big(\frac{\operatorname{cos} \left( \mathbf{r}_i(\mathbf{h}), \mathbf{q}_i^y \right)}{\tau}\Big)
    }{
        \sum_{\hat{y} \in \mathcal{C}} \exp \Big(\frac{\operatorname{cos} \left( \mathbf{r}_i(\mathbf{h}), \mathbf{q}_i^{\hat{y}} \right)}{\tau}\Big)
    }\Bigg].
\end{equation}
This contrastive loss function is designed to align the client's local representation to its corresponding global personalized prototype. This helps to reduce the discrepancy between local and global knowledge updates. Cosine similarity is used to compute the distance among global compact embedding and current local compact.

Apart from aligning global knowledge, to enable the local model to learn more inter-client knowledge, we also define a local prototype-based loss term, $\mathcal{L}_{c}$, as follows:
\begin{equation}
    \label{eqn_9}
    \mathcal{L}_{c} = \frac{1}{N} \sum_{j\in N} \mathbb{E}_{(\mathbf{h}, y) \in \mathcal{D}_i} \Bigg[-\log\frac{\exp\Big(\frac{\operatorname{cos} \left(\mathbf{r}_i(\mathbf{h}), \mathbf{p}_j^y\right)}{\tau}\Big)}{\sum_{\hat{y} \in \mathcal{C}} \exp \Big(\frac{\operatorname{cos} \left( \mathbf{r}_i(\mathbf{h}), \mathbf{p}_j^{\hat{y}} \right)}{\tau} \Big)}\Bigg].
\end{equation}
In addition to the representation learning terms, we also define a classification term, denoted as $\mathcal{L}_{ce}$, which employs a cross-entropy function for classifier learning.

To stabilize training and guide the model through a logical progression of learning stages, preventing it from being overwhelmed by complex objectives too early, we introduce a warm-up coefficient $\lambda$ inspired by curriculum learning~\cite{bengio2009curriculum}. It is defined as 
\begin{equation}
    \label{eqn_10}
    \begin{aligned}
    \lambda &= \lambda_{min} + \frac{(\lambda_{\max} - \lambda_{min})}{2} \left(1 - \cos\left(\pi \cdot u_t \right)\right), \\
    u_t &= \frac{\min\left(t, T_{warm}\right)}{T_{warm}},
    \end{aligned}
\end{equation}
where $\lambda_{min}$ and $\lambda_{\max}$ are the minimum and maximum values of $\lambda$, $t$ represents the $t$-th communication round, while the hyperparameter $T_{warm}$ controls the length of the representation learning warm-up phase. When $t\ge T_{\mathrm{warm}}$, $\lambda_t\equiv \lambda_{\max}$. The default values of $\lambda_{\min}$, $\lambda_{\max}$, and $T_{warm}$ are set to zero, one, and 50, respectively.

At the early stage, $\lambda$ is close to zero. The total loss is therefore dominated by the cross-entropy term $\mathcal{L}_{ce}$. Hence, the model's first and most important job is to learn the basics of the classification task. By focusing almost exclusively on cross-entropy, the model learns the coarse high-level decision boundaries. As training progresses, $\lambda$ smoothly increases from zero to one through cosine-based scheduler. The contrastive loss terms $\mathcal{L}_g$ and $\mathcal{L}_c$ become more and more influential. This curriculum introduces the harder task for the model learning: organizing the feature space. This feature organizing pressure from the contrastive loss forces the model to build a more robust, meaningful, and generalizable understanding of the feature from both local and global knowledge. \begin{algorithm}[t]
\caption{Server}
\label{alg:fedapa-server}
\begin{algorithmic}[1]
\Require Clients $i =\{1,\dots,N\}$, rounds $T$
\State Initialize prototype set $\mathbf{P}_{i, 0}$ and $\mathbf{Q}_{i, 0}$ for each client $i$
\For{$t=1$ to $T$}
  \For{each client $i$ in parallel}
    \State $\mathbf{P}_{i, t} \leftarrow$ ClientUpdate$(i,\,\mathbf{Q}_{i, t-1},\,\mathbf{P}_{t-1})$   
    \For{each $c\in\C_i$ with $c\in\C_j$}
    \State Compute cosine similarity $s_{ij}^{c}$ by Eq.~\eqref{eqn_5}
    \State Convert to weights $\alpha_{ij}^{c}$ by Eq.~\eqref{eqn_6}
    \State Form $\mathbf{q}_{i,t}^{c}$ by Eq.~\eqref{eqn_7}
    \EndFor
    \State Update $\mathbf{Q}_{i, t} \gets \{\mathbf{q}_{i,t}^{c}\}_{c\in\C_i}$
  \State Pad missing prototypes in $\mathbf{Q}_{i, t}$ and $\mathbf{P}_{i, t}$ by Eq.~\eqref{eq:padding}
  \EndFor
  \State \textbf{Return} to each client $i$: $\mathbf{Q}_{i, t}$ and the fully uploaded prototype sets $\mathbf{P}_{t}$
\EndFor
\end{algorithmic}
\end{algorithm}By the time $\lambda$ reaches its maximum value, the model is being trained to do both jobs simultaneously: classify correctly and maintain a well-structured internal representation. The total loss for client $i$ is computed by combining the weighted representation loss terms, $\lambda (\mathcal{L}_{g} + \mathcal{L}_{c})$ and $\mathcal{L}_{ce}$, through addition, as shown below: 
\begin{equation}
    \label{eqn_11}
    \begin{aligned}
    \mathcal{L}(\mathcal{D}_i, w_i)
    &= \mathcal{L}_{ce}\big(w_i\big) \\
    &\quad + \lambda \Big( 
        \mathcal{L}_{g}\big(w_i; \mathbf{Q}_i\big)
        + \mathcal{L}_{c}\big(w_i; \mathbf{P}\big)
    \Big).
    \end{aligned}
\end{equation}
Eq.~\eqref{eqn_11} modifies the original prototype-based FL loss Eq.~\eqref{eqn_3} by: (i) replace the general global prototype $\mathbb{P}$ with the personalized global prototype $\mathbf{Q}_i$; (ii) add additional inter-client prototype learning loss; (iii) using warm-up schedule for $\lambda$. After finishing local training of each communication round, all clients will upload their prototype sets to the PS for the next round.

\begin{algorithm}[t]
\caption{ClientUpdate$(i,\,\mathbf{Q}_{i, t - 1},\,\mathbf{P}_t - 1)$}
\label{alg:fedapa-client}
\begin{algorithmic}[1]
\Require Local dataset $\mathcal{D}_i$; model $w_i$; steps $S$;
\State Receive $\mathbf{Q}_{i, t - 1}$, and $\mathbf{P}_{t - 1}$
\State Update warm-up coefficient $\lambda_t$ by Eq.~\eqref{eqn_10}
\For{$s=1$ to $S$}
  \For{each batch $(\mathbf{h},y)$ in $\mathcal{D}_i$}
    \State Compute $\mathcal{L}_{g}$ by Eq.~\eqref{eqn_8} with personalzied global prototypes $\mathbf{Q}_{i, t - 1}$
    \State Compute $\mathcal{L}_c$ by Eq.~\eqref{eqn_9} using fully uploaded prototype sets $\mathbf{P}_{t - 1}$
  \State Update $w_{i, t}$ by Eq.~\eqref{eqn_11}
  \EndFor
\EndFor
\State Build local prototypes $\mathbf{p}_{i,t}^{c}$ for each $c\in\Ci$ by Eq.~\eqref{eqn_4} to form $\mathbf{P}_{i, t}$
\State \Return $\mathbf{P}_{i, t}$ to the server
\end{algorithmic}
\end{algorithm}

\section{Convergence Analysis}
\label{sec:convergence}
We summarize the main theoretical guarantees of FedAPA in this section, and refer to existing works~\cite{fedproto, ghadimi2013stochastic, lee2023implicit, newhouse2025softmax} for the following assumptions and results.

\subsection{Assumptions}
We assume full participation and let $L_{\inf}\in\mathbb R$ be a uniform lower bound for all local objectives.

\begin{assumption}[Smoothness and lower boundedness]
\label{ass:A1}
For each round $t$ and client $i$, the local objective $\mathcal L_{i,t}$ in
\eqref{eqn_11} has $L_t$–Lipschitz gradient:
\begin{equation}
\|\nabla\mathcal L_{i,t}(u) - \nabla\mathcal L_{i,t}(v)\|
\le L_t\|u-v\|,
\quad \forall u,v,
\end{equation}
with
\begin{equation}
L_t = L_{\mathrm{ce}} + \lambda_t L_\Phi,
\qquad
L_t \le L_{\max} := L_{\mathrm{ce}} + \lambda L_\Phi
\end{equation}
after warm-up. Moreover,
\begin{equation}
\mathcal L_{i,t}(w) \ge L_{\inf}
\quad \forall i,t,w.
\end{equation}
\end{assumption}

\begin{assumption}[Stochastic gradients]
\label{ass:A2}
For client $i$, round $t$, and local step $s$, the stochastic gradient $g_{i,t,s}$ satisfies
\begin{align}
\mathbb E[g_{i,t,s} \mid w_{i,t,s}]
&= \nabla \mathcal L_{i,t}(w_{i,t,s}), \\
\mathbb E\|g_{i,t,s} - \nabla \mathcal L_{i,t}(w_{i,t,s})\|^2
&\le \sigma^2,
\end{align}
for some $\sigma^2 \ge 0$, uniformly over all $i,t,s$. In addition, there exists $G>0$ such that
\begin{equation}
\mathbb E\bigl[\|g_{i,t,s}\|^2\bigr] \le G^2
\qquad \text{for all } i,t,s.
\end{equation}
\end{assumption}

\begin{assumption}[Bounded and Lipschitz representations]
\label{ass:A3}
Encoder outputs are uniformly bounded:
\begin{equation}
\|\mathbf r_{w^\theta}(\mathbf h)\| \le 1
\quad \text{for all } w^\theta,\mathbf h,
\end{equation}
so all prototypes (original and padded) satisfy $\|\mathbf p_{i,t}^c\|\le 1$.
The encoder is Lipschitz in its parameters: there exists $L_{w^\theta}>0$ such that
\[
\|\mathbf r_{w^\theta}(\mathbf h) - \mathbf r_{\hat w^\theta}(\mathbf h)\|
\le L_{w^\theta}\,\|w^\theta-\hat w^\theta\|,
\quad \forall w^\theta,\hat w^\theta,\mathbf h.
\]
\end{assumption}

\begin{assumption}[Lipschitz personalized aggregation]
\label{ass:A4}
Let $\mathbf P$ be the stacked local prototypes, and let
$\mathbf Q_\tau(\mathbf P)$ be the similarity-weighted aggregation in
Section~\ref{subsec:agg}. There exists $L_{\mathrm{agg}}(\tau)>0$ such that
\begin{equation}
\|\mathbf Q_\tau(\mathbf P) - \mathbf Q_\tau(\mathbf P')\|_F
\le L_{\mathrm{agg}}(\tau)\,\|\mathbf P - \mathbf P'\|_F
\quad \forall\,\mathbf P,\mathbf P'.
\end{equation}
\end{assumption}

\begin{assumption}[Prototype regularizer]
\label{ass:A5}
For each client $i$, the prototype-based regularizer
$\Phi_i(w;\mathbf Q,\mathbf P)$ is Lipschitz in $(\mathbf Q,\mathbf P)$ and
uniformly bounded: there exist $c_\Phi(\tau)>0$ and $B_\Phi>0$ such that
\begin{equation}
\begin{aligned}
|\Phi_i(w;\mathbf Q,\mathbf P)-\Phi_i(w;\mathbf Q',\mathbf P')|
&\le c_\Phi(\tau)\big(\|\mathbf Q-\mathbf Q'\|_F \\
&\quad 
+ \|\mathbf P-\mathbf P'\|_F\big),
\end{aligned}
\end{equation}
\begin{equation}
0\le \Phi_i(w;\mathbf Q,\mathbf P) \le B_\Phi,
\end{equation}
for all $w,\mathbf Q,\mathbf P,\mathbf Q',\mathbf P'$.
\end{assumption}
\begin{assumption}[Warm-up schedule]
\label{ass:A6}
The weights $\{\lambda_t\}_{t\ge 1}$ are nondecreasing and there exist
$T_{\mathrm{warm}}<\infty$ and $\lambda\ge 0$ such that
\begin{equation}
\lambda_t = \lambda
\quad \text{for all } t\ge T_{\mathrm{warm}}.
\end{equation}
\end{assumption}

For each client $i$ and round $t$, we define the within-round gradient norm
\begin{equation}
G_{i,t}^2 := \sum_{s=0}^{S-1}
\mathbb E\bigl[\|\nabla \mathcal L_{i,t}(w_{i,t,s})\|^2\bigr].
\end{equation}
We also define
\begin{equation}
\Gamma(\tau) := c_\Phi(\tau)\bigl(1 + L_{\mathrm{agg}}(\tau)\bigr),
\end{equation}
and denote by $k_{\mathrm{glob}}$ the prototype-movement constant from Lemma 2(b) in Appendix A.

\subsection{Convergence results}

Based on the above assumptions, we have the following convergence result for an arbitrary client in FedAPA.

\begin{theorem}[One-round deviation for an arbitrary client]
\label{thm:one-round}
Let Assumptions~\ref{ass:A1} to \ref{ass:A5} hold.
For any client $i$ and any $t$, with step size
$0<\eta\le 1/L_t$, we have
\begin{align}
\mathbb E\big[\mathcal L_{i,t+1}(w_{i,t+1,0})\big]
&\le
\mathbb E\big[\mathcal L_{i,t}(w_{i,t,0})\big]
- \tfrac{\eta}{2} G_{i,t}^2
\nonumber\\
&\quad
+ \tfrac{L_t\eta^2}{2}S\sigma^2
+ \lambda_{t+1}\,\Gamma(\tau)\,k_{\mathrm{glob}}\,\eta\,S\,G
\nonumber\\
&\quad
+ |\lambda_{t+1}-\lambda_t|\,B_\Phi.
\label{eq:one-round-dev}
\end{align}
\end{theorem}

Under Assumptions~\ref{ass:A1} to \ref{ass:A5}, Theorem~\ref{thm:one-round} bounds the one-round change of any client's objective as a descent term, proportional to the squared gradient norm, minus three error terms: a stochastic-variance term, a prototype-refresh term of order $\lambda_{t+1}\Gamma(\tau)k_{\mathrm{glob}}\eta S G$, and a schedule-change term of order $|\lambda_{t+1}-\lambda_t|B_\Phi$. When $\eta$, $\lambda_t$, and $\tau$ are chosen so that these errors are smaller than the descent term, the expected objective of each client is nonincreasing up to $O(\eta^2)$ and prototype-refresh fluctuations. Similarity-weighted aggregation and the warm-up schedule affect per-round descent only through $\Gamma(\tau)$, $k_{\mathrm{glob}}$, and $\lambda_t$, which is consistent with standard nonconvex SGD analyses and prototype-based FL results; intuitively, smoother or similarity-aware aggregation can reduce $\Gamma(\tau)$ and thus relax the per-round descent condition relative to a large fixed $\lambda$ with uniform weights.

Building on Theorem~\ref{thm:one-round}, we have the following overall convergence guarantee after the warm-up phase.

\begin{theorem}[$\varepsilon$-stationarity after warm-up]
\label{thm:epsilon-stationarity-correct}
Suppose Assumptions~\ref{ass:A1} to \ref{ass:A6} hold and that
$\lambda_t \equiv \lambda$ and $L_t \le L_{\max}$ for all $t \ge T_{\mathrm{warm}}$. We consider $T$ communication rounds after $T_{\mathrm{warm}}$; for notational convenience, we re-index these rounds as $t \in \{1,\dots,T\}$. Each round has $S$ local steps, so there are $K := T S$ total local updates in this phase. For any client $i$, the iterates of FedAPA satisfy
\begin{equation}
\begin{aligned}
\frac{1}{K}\sum_{t=1}^{T}\sum_{s=0}^{S-1}
\mathbb{E}\bigl[\|\nabla \mathcal L_{i,t}(w_{i,t,s})\|^2\bigr]
\;&\le\;
\frac{2\Delta_i}{K\eta}
+ L_{\max}\eta\sigma^2 \\
&\quad
+ 2\lambda\,\Gamma(\tau)\,k_{\mathrm{glob}}\,G,
\end{aligned}
\label{eq:avg-grad-bound-correct}
\end{equation}
for any stepsize $0<\eta\le 1/L_{\max}$, where
\begin{equation}
\Delta_i := \mathbb{E}\big[\mathcal L_{i,T_{\mathrm{warm}}}(w_{i,T_{\mathrm{warm}},0})\big]
- L_{\inf}.
\end{equation}

Moreover, for any target $\varepsilon>0$, if the parameters satisfy
\begin{align}
0 < \eta
&\le
\min\Big\{\frac{1}{L_{\max}},\; \frac{\varepsilon}{3L_{\max}\sigma^2}\Big\},
\label{eq:cond-eta-correct}\\[4pt]
\lambda
&\le
\frac{\varepsilon}{6\,\Gamma(\tau)\,k_{\mathrm{glob}}\,G},
\label{eq:cond-lambda-correct}\\[4pt]
T
&\ge
\frac{6\Delta_i}{\varepsilon\,\eta\,S},
\label{eq:cond-T-correct}
\end{align}
then FedAPA achieves $\varepsilon$-stationarity for client $i$:
\begin{equation}
\frac{1}{K}\sum_{t=1}^{T}\sum_{s=0}^{S-1}
\mathbb{E}\bigl[\|\nabla \mathcal L_{i,t}(w_{i,t,s})\|^2\bigr]
\;\le\; \varepsilon.
\end{equation}
\end{theorem}

Theorem~\ref{thm:epsilon-stationarity-correct} shows that, after warm-up, $\lambda_t\equiv\lambda$, the time-averaged gradient norm for any client is bounded by an initial gap term $O(1/(K\eta))$, a variance term $O(\eta)$, and a prototype-coupling term proportional to $\lambda\,\Gamma(\tau)\,k_{\mathrm{glob}}\,G$. For any target $\varepsilon>0$, the theorem gives explicit conditions on $\eta,\lambda,\tau$ and the number of rounds $T$ ensuring this bound is at most $\varepsilon$; the asymptotic error floor is of order $\lambda\,\Gamma(\tau)\,k_{\mathrm{glob}}\,G$, so larger $\Gamma(\tau)$ or $k_{\mathrm{glob}}$ require smaller $\lambda$. The influence of similarity-weighted prototypes and warm-up is fully captured by the constants $\Gamma(\tau)$, $k_{\mathrm{glob}}$, and $\lambda$; in practice, smaller $\tau$ tends to increase $\Gamma(\tau)$ and thus tighten the admissible range of $\lambda$, while moderate $\tau$ can mitigate this effect. With standard stepsize choices (e.g., $\eta$ on the order of $\varepsilon$ or $1/\sqrt{K}$), the bound recovers the usual $O(1/\sqrt{K})$ rate for nonconvex stochastic methods, up to constants from prototype coupling. We refer the reader to Appendix A for detailed proofs.

\begin{figure}[htbp]
    \centering
    \includegraphics[width=\columnwidth]{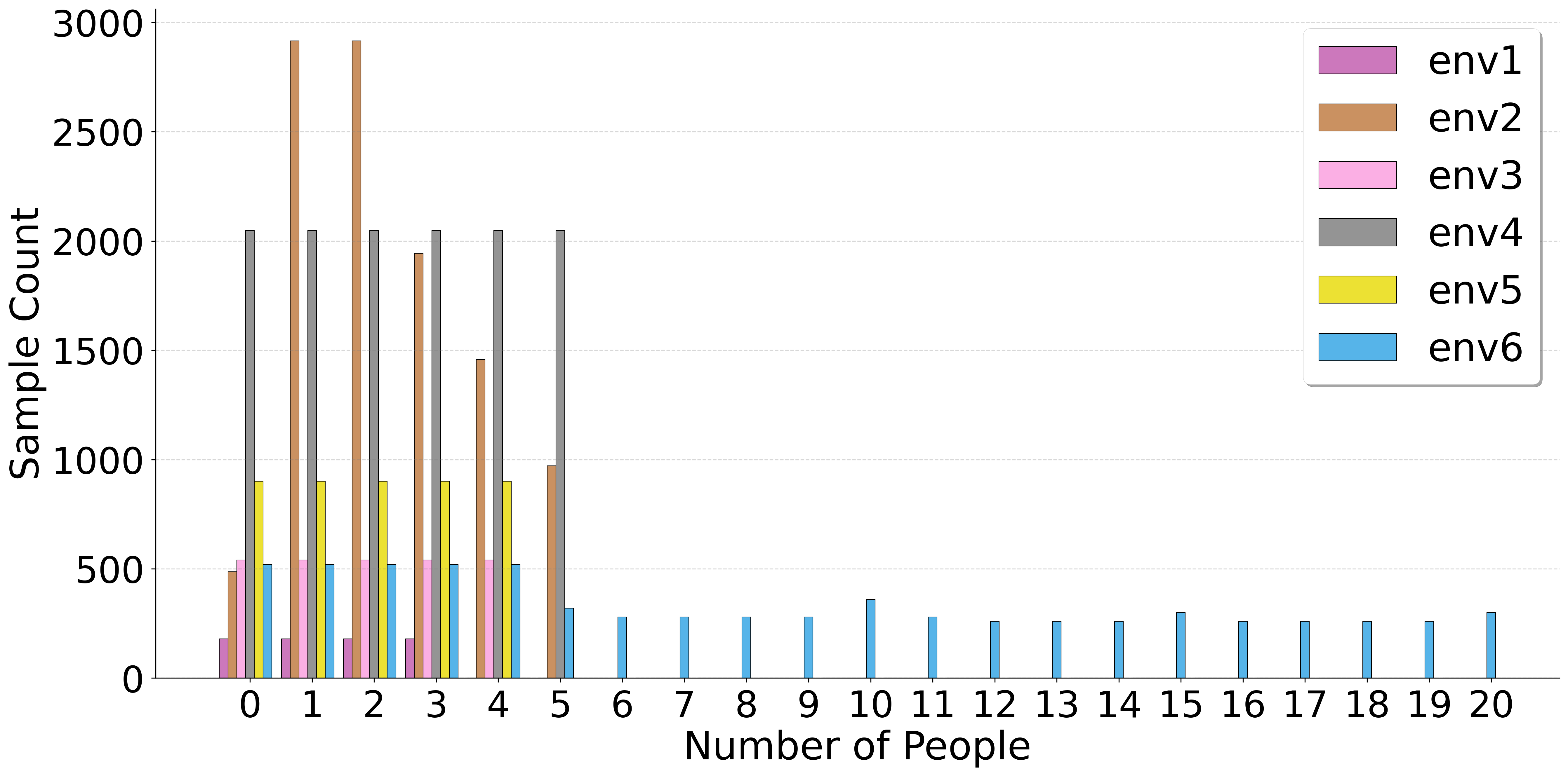}
    \caption{Data distribution of six distinct environments.}
    \label{fig:pic4}
\end{figure}
\begin{table}[htbp]
\centering
\caption{Architecture overview and layer configurations for TinyConvNet4, MiddleConvNet4, and LargeConvNet4 models. All convolutional layers use $3\times 3$ kernels with stride 2 unless otherwise specified. 'C' denotes the number of output channels.}
\label{tab:architecture_comparison}
\resizebox{\columnwidth}{!}{%
\begin{tabular}{c|c|ccc}
\toprule
\multicolumn{2}{c}{\textbf{Layer Information}} & \multicolumn{3}{c}{\textbf{Model Configurations}} \\
\cmidrule(lr){1-2} \cmidrule(lr){3-5}
\textbf{Layer Block} & \textbf{Layer Type} & \multicolumn{1}{c}{\textbf{TinyConvNet4}} & \multicolumn{1}{c}{\textbf{MiddleConvNet4}} & \multicolumn{1}{c}{\textbf{LargeConvNet4}} \\
& & \textbf{(C / Shape)} & \textbf{(C / Shape)} & \textbf{(C / Shape)} \\
\midrule
    \multirow{3}{*}{\textbf{Block 1}} & Conv2d & 256 / (500, 121) & 16 / (500, 121) & 16 / (500, 121) \\
    & BatchNorm2d & 256 & 16 & 16 \\
    & ReLU & - & - & - \\\midrule
    \multirow{3}{*}{\textbf{Block 2}} & Conv2d & - & 32 / (250, 61) & 32 / (250, 61) \\
    & BatchNorm2d & - & 32 & 32 \\
    & ReLU & - & - & - \\\midrule
    \multirow{3}{*}{\textbf{Block 3}} & Conv2d & - & - & 64 / (125, 31) \\
    & BatchNorm2d & - & - & 64 \\
    & ReLU & - & - & - \\\midrule
    \multirow{3}{*}{\textbf{Block 4}} & Conv2d & - & - & 128 / (63, 16) \\
    & BatchNorm2d & - & - & 128 \\
    & ReLU & - & - & - \\\midrule
    \multirow{3}{*}{\textbf{Block 5}} & Conv2d & - & - & 256 / (32, 8) \\
    & BatchNorm2d & - & - & 256 \\
    & ReLU & - & - & - \\\midrule
    \textbf{Pooling} & AdaptiveAvgPool2d & (256, 1, 1) & (32, 1, 1) & (256, 1, 1) \\\midrule
    	\textbf{Head} & Conv2d $(1\times 1, s=1)$ & - & 256 / (1, 1) & 256 / (1, 1) \\\midrule
    \textbf{Classifier} & Linear & 20 & 20 & 20 \\\midrule
    \textbf{Total Params} & - & \textbf{7.96 K} & \textbf{18.44 K} & \textbf{463.75 K} \\\midrule
    \textbf{FLOPs (MFLOPS)} & - & \textbf{340.75} & \textbf{162.85} & \textbf{606.35} \\
\bottomrule
\end{tabular}%
}
\begin{tablenotes}
\footnotesize
\item Note: '-' indicates the layer/block is not present in that model variant.
\end{tablenotes}
\end{table}

\section{Experimental Evaluation}
In this section, we evaluate our proposed FedAPA method using a real-world distributed Wi-Fi CSI people-counting dataset collected across six distinct scenarios, including living room, classroom, bus, conference room, and hotel room. We first introduce the dataset and model architectures used in our experiments. Then, we present the implementation details, including hyperparameter settings and evaluation metrics. Finally, we compare our method with several FL methods and conduct extensive ablation studies to validate the effectiveness of each component in our proposed framework.

\subsection{Datasets and Models}
Fig.~\ref{fig:pic4} shows the heterogeneous data distribution across six distinct environments, ranging between 0 and 20 people. \begin{figure}[htbp]
    \centering
    \includegraphics[width=\columnwidth]{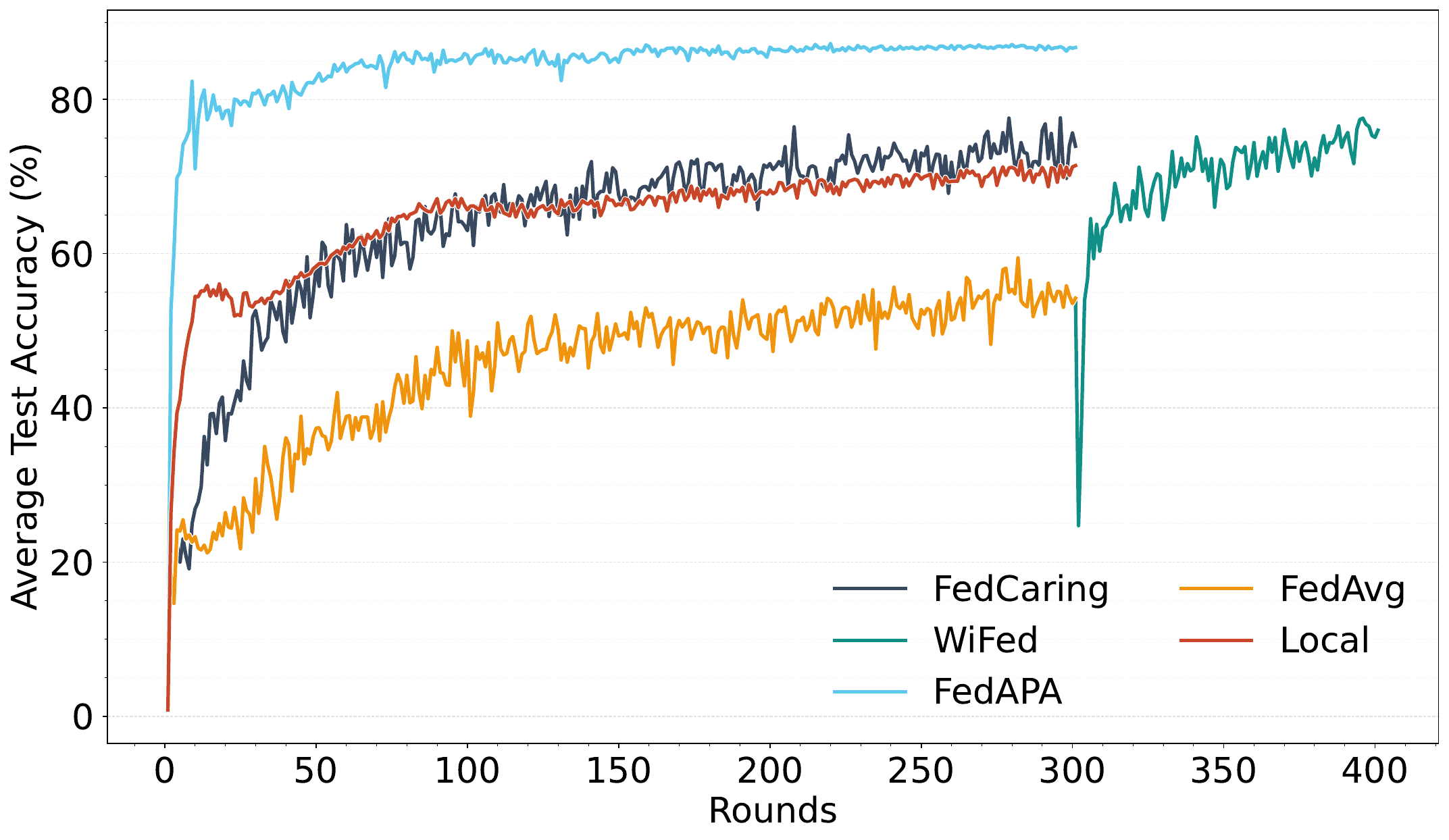}
    \caption{The convergence behavior of different methods in heterogeneous data.}
    \label{fig:pic6}
\end{figure}
\begin{figure}[htbp]
    \centering
    \includegraphics[width=\columnwidth]{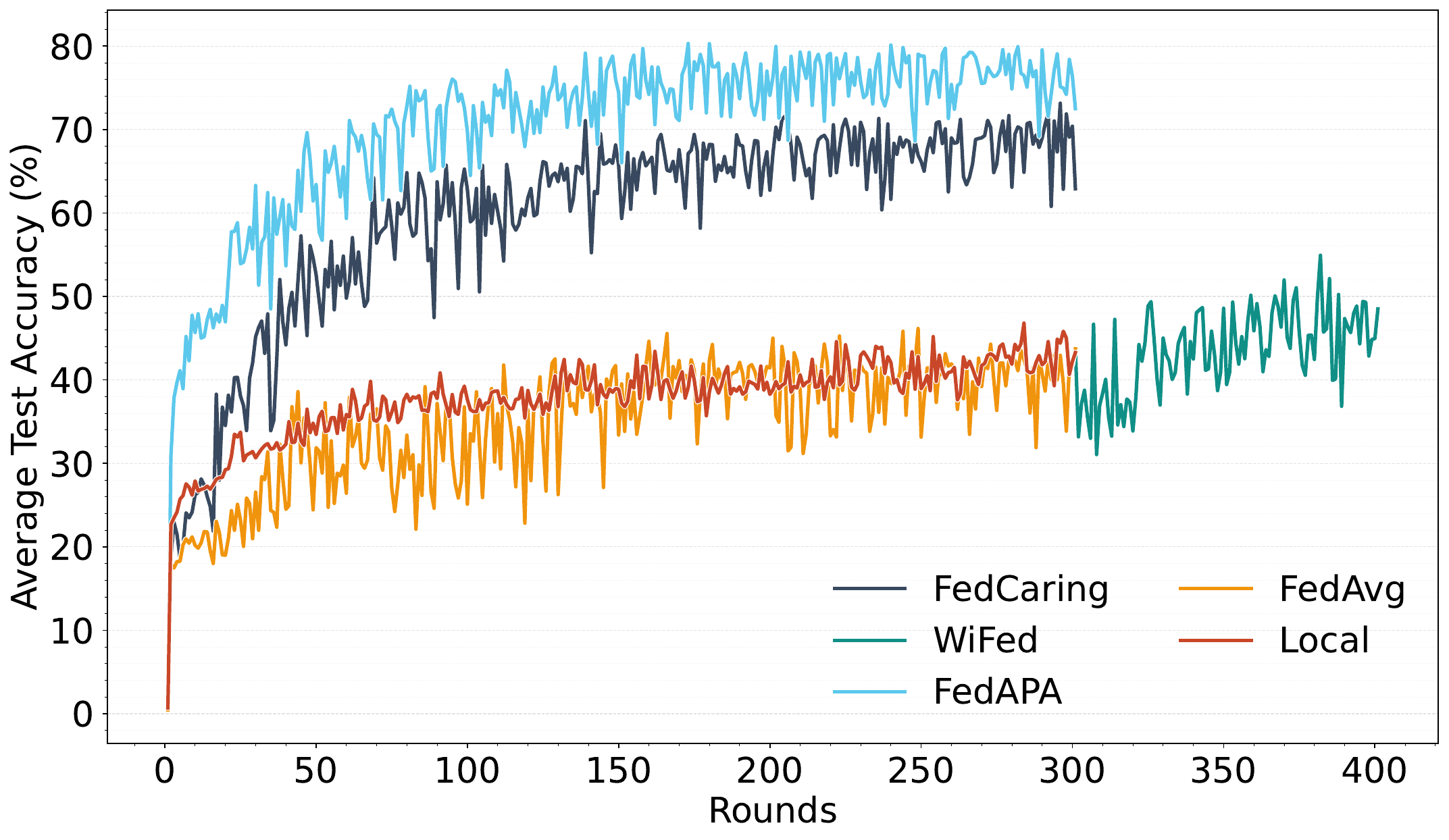}
    \caption{The convergence behavior of different methods in heterogeneous model.}
    \label{fig:pic7}
\end{figure}
Two of the six distinct environments for collecting CSI-based crowd counting data: a living room and a conference room, are depicted in Figs.~\ref{fig:pic1} (a) and (b). To simulate model heterogeneity in the Wi-Fi CSI-based crowd counting task where resource constraints exist, we design three CNN architectures and assign them based on the training data of each client. Table~\ref{tab:architecture_comparison} presents the number of layers, parameters, and computational costs for each model. By default, we use LargeConvNet4 in the subsequent experiments, except in the model heterogeneity setting experiments.

\subsection{Data Preprocessing}
We first filter out null subcarriers as they introduce outliers and degrade learning performance\cite{guo2025rssi}. Then a sliding window with a size of 1,000 without overlapping is applied to segment the CSI data along the time axis. Each segment is then transformed into a 1,000$\times$242 matrix, where 1,000 represents the number of CSI samples and 242 represents the number of subcarriers.

\subsection{Baselines}
We compare our method with local training and various FL methods that applied in Wi-Fi sensing tasks. Among them, WiFederated\cite{wifederated} fine-tuned the global model trained based on FedAvg\cite{fedavg} to build personalized models for each client. CARING\cite{caring} proposed a weight adaptation scheme that automatically assigns two different weights (i.e., 0.3 and 1) to each local model during model aggregation.
\begin{table}[htbp]
\centering
\caption{The overall performance comparison of different methods in both heterogeneous statistical data and model architecture settings.}\label{tab:tab1}
\resizebox{\columnwidth}{!}{%
\begin{tabular}{c|ccc|ccc}
\toprule
    \textbf{Heterogeneous Setting} & \multicolumn{3}{c|}{\textbf{Statistical Data}} & \multicolumn{3}{c}{\textbf{Model Architecture}} \\\midrule
    \textbf{Metric} & \textbf{Acc.} & \textbf{F1} & \textbf{MAE} & \textbf{Acc.} & \textbf{F1} & \textbf{MAE} \\\midrule
    Local & 72.06 & 71.32 & 0.69 & 46.81 & 39.45 & 1.26 \\\midrule
    FedAvg\cite{fedavg} & 59.46 & 38.26 & 2.01 & 46.16 & 37.51 & 2.03 \\\midrule
    WiFed\cite{wifederated} & 77.57 & \uline{76.91} & \uline{0.57} & 54.94 & 47.80 & 1.34 \\\midrule
    FedCaring\cite{caring} & \uline{77.60} & 69.18 & 0.62 & \uline{70.00} & \uline{68.43} & \uline{0.77} \\\midrule
    FedAPA & \textbf{87.25} & \textbf{85.91} & \textbf{0.23} & \textbf{80.31} & \textbf{78.94} & \textbf{0.48} \\
\bottomrule
\end{tabular}%
}
\end{table}

\subsection{Implementation Details}

The client participation ratio $\rho$ controls the fraction of clients sampled each round. When the number of clients is large, a smaller $\rho$ reduces bandwidth per round; in our setting with few clients, we use full participation and set $\rho=1.0$. The number of local epochs $E$ determines the interval between communication rounds; we set $E=1$ for all experiments.

For local training we use stochastic gradient descent (SGD) with batch size $16$, learning rate $10^{-2}$, momentum $0.5$, and weight decay $10^{-5}$. To reflect a personalized FL use case, evaluation is performed on each client's local test split. We report Accuracy, F1, and mean absolute error (MAE); F1 is included to handle class imbalance \cite{imbalancedata}. To reduce variance, we average metrics over the last five communication rounds in all tables. FedAPA results are highlighted in \textbf{bold}, and the best results among those baselines are \uline{underlined}. Each client uses a local split with $80\%$ of data for training and $20\%$ for testing. Experiments are run with PyTorch 2.7.1 on a server with two Intel Xeon Silver 4210R CPUs (20 cores), 128,GB RAM, and four NVIDIA RTX,3090 GPUs, running Ubuntu 24.04.

\subsection{Performance Evaluation and Analysis}
In this section, we compare FedAPA with multiple baselines, including traditional FL methods and FL methods that focus on statistical data. Hyperparameter settings of those baseline approaches are configured according to prior works\cite{wifederated,caring}. Extensive experiments were conducted to provide a comprehensive analysis of our methods.

\subsubsection{Practical heterogeneity setting}
The results of all algorithms in the practical heterogeneity setting are shown in Table \ref{tab:tab1}, and their convergence behavior is illustrated in Fig.~\ref{fig:pic6} and Fig.~\ref{fig:pic7}. Here we use data and model heterogeneity settings to have a comprehensive evaluation of our methods in heterogeneous Wi-Fi CSI-based crowd counting tasks. Compared with other baselines, FedAPA is the top-performing method in both model heterogeneity and data heterogeneity setting. For the data heterogeneity setting, FedAPA increases the local accuracy, F1, and MAE by 15.19\%, 14.59\%, and 0.46 respectively. In the model heterogeneity setting, FedAPA increases the local accuracy, F1, and MAE by 33.5\%, 39.49\%, and 0.78 respectively. It is worth noting that the FedAPA method has worse sensing performance in the model heterogeneity setting. Those methods that include simple model parameter averaging also encounter performance degradation in the model heterogeneity scenario, which suggests a more sophisticated model aggregation mechanism during the global model generation. In this practical heterogeneity setting, the local data distributions and model architectures across all clients become imbalanced. Consequently, the F1 metrics of all algorithms are inferior to their accuracies. \begin{figure}[htbp]
    \centering
    \subfloat[Impact of feature dimensions]{\includegraphics[width=0.47\columnwidth]{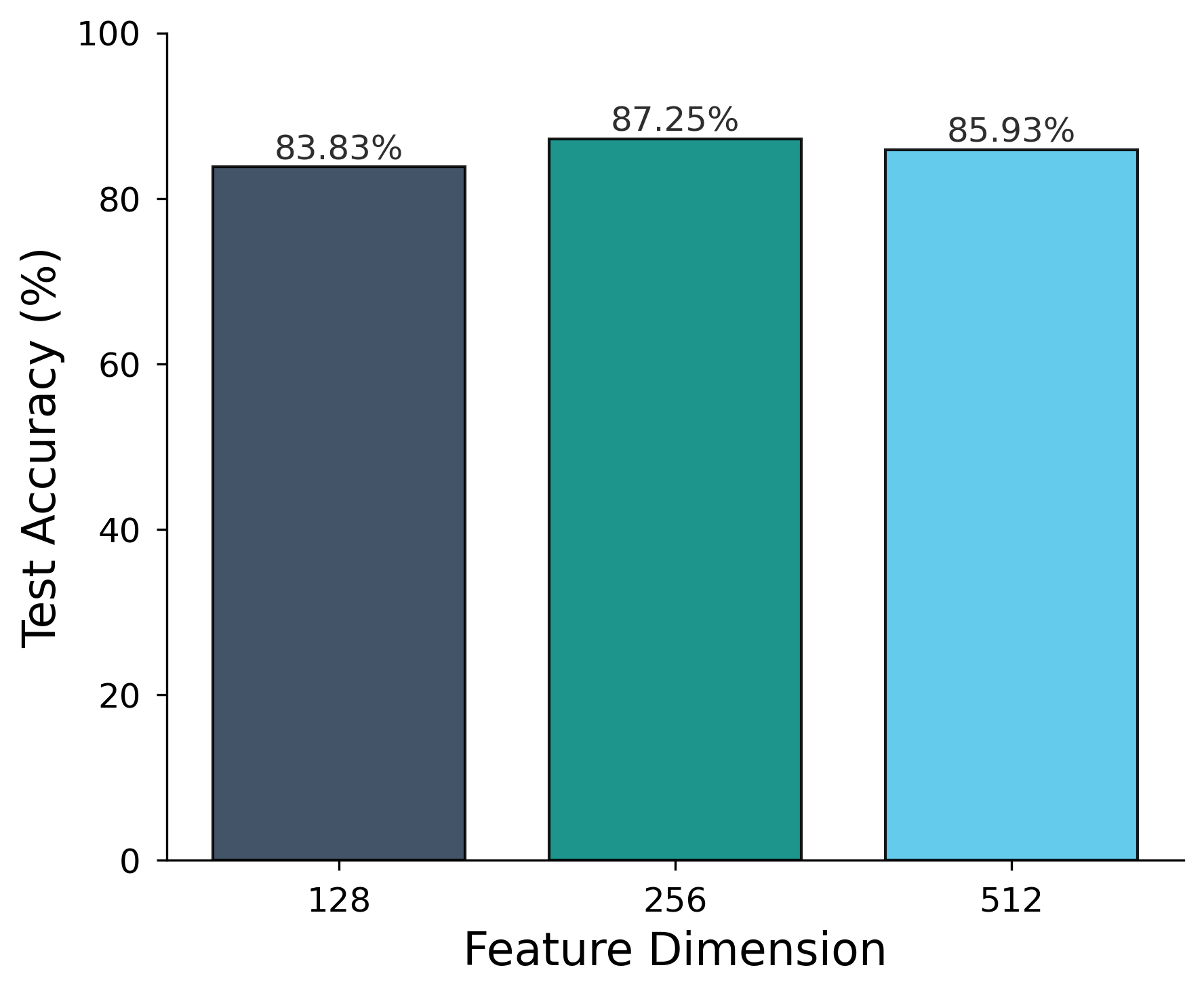}\label{fig:pic8a}}\hfil
    \subfloat[Impact of warm-up rounds]{\includegraphics[width=0.47\columnwidth]{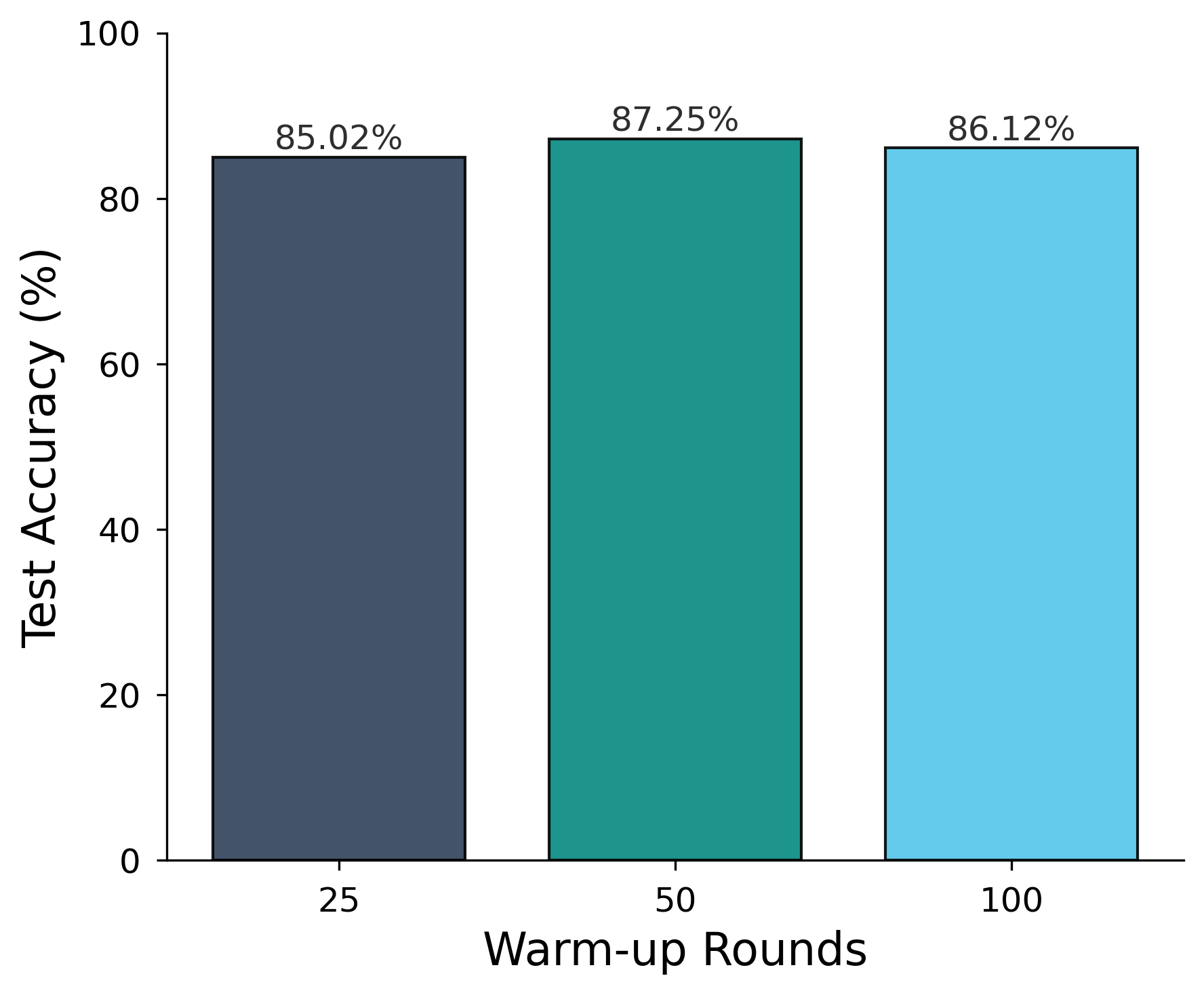}\label{fig:pic8b}}
    \caption{Impact of \textbf{(a)} feature dimensions and \textbf{(b)} warm-up rounds on FedAPA performance.}
    \label{fig:pic8}
\end{figure}
Moreover, simplely averaging the model parameters may introduce significant local drift updates in heterogeneous model scenarios. FedAvg performs worse than the majority of other baselines. After fine-tuning over several epochs using either original or re-balanced local data, the performance is remarkably improved, as demonstrated by the WiFed. This suggests that fine-tuning from a pre-trained representation can help improve the Wi-Fi sensing performance under heterogeneous settings. FedAPA further outperforms WiFed in either data or model heterogeneity settings, even without fine-tuning from local data, which demonstrates its superior capability in tackling heterogeneous tasks. The performance of FedCaring also suggests that by carefully assigning weights to local models during aggregation, the negative impact of data and model heterogeneity can be alleviated. However, FedCaring still falls short of FedAPA, which indicates that simply assigning two different weights to local models is insufficient to capture the complex relationships among clients.

\subsubsection{Impact of different feature dimensions for embeddings} 
As conise score is used to measure the similarity between embeddings and prototypes, the feature dimension of embeddings may affect the performance of FedAPA. We conduct experiments with different feature dimensions, including 128, 256, and 512. The results are shown in Fig.~\ref{fig:pic8a}. It can be observed that the performance of FedAPA improves as the feature dimension increases from 128 to 256. However, when the feature dimension is further increased to 512, the performance slightly decreases. This suggests that a moderate feature dimension is beneficial for capturing the essential characteristics of the data while avoiding overfitting or redundancy in the embeddings. Therefore, we choose 256 as the default feature dimension for our experiments.

\subsubsection{Impact of different warm-up rounds}
Another important hyperparameter is the number of warm-up rounds $T_{warm}$, which controls the length of the representation learning warm-up phase. \begin{figure}[htbp]
    \centering
    \subfloat[Impact of different temperatures.]{\includegraphics[width=0.47\columnwidth]{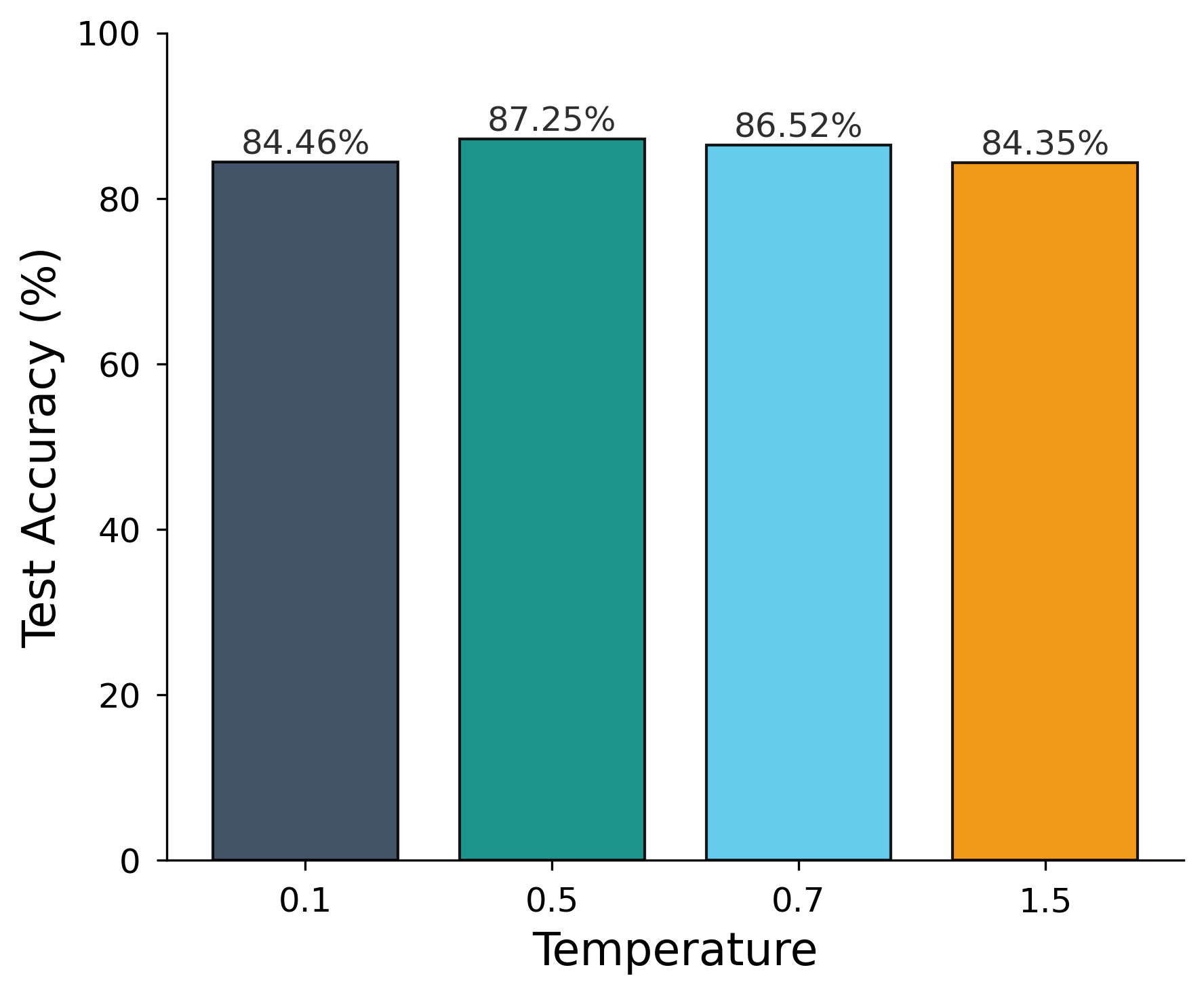}\label{fig:pic9a}}\hfil
    \subfloat[Impact of warm-up based $\lambda$ coefficient.]{\includegraphics[width=0.47\columnwidth]{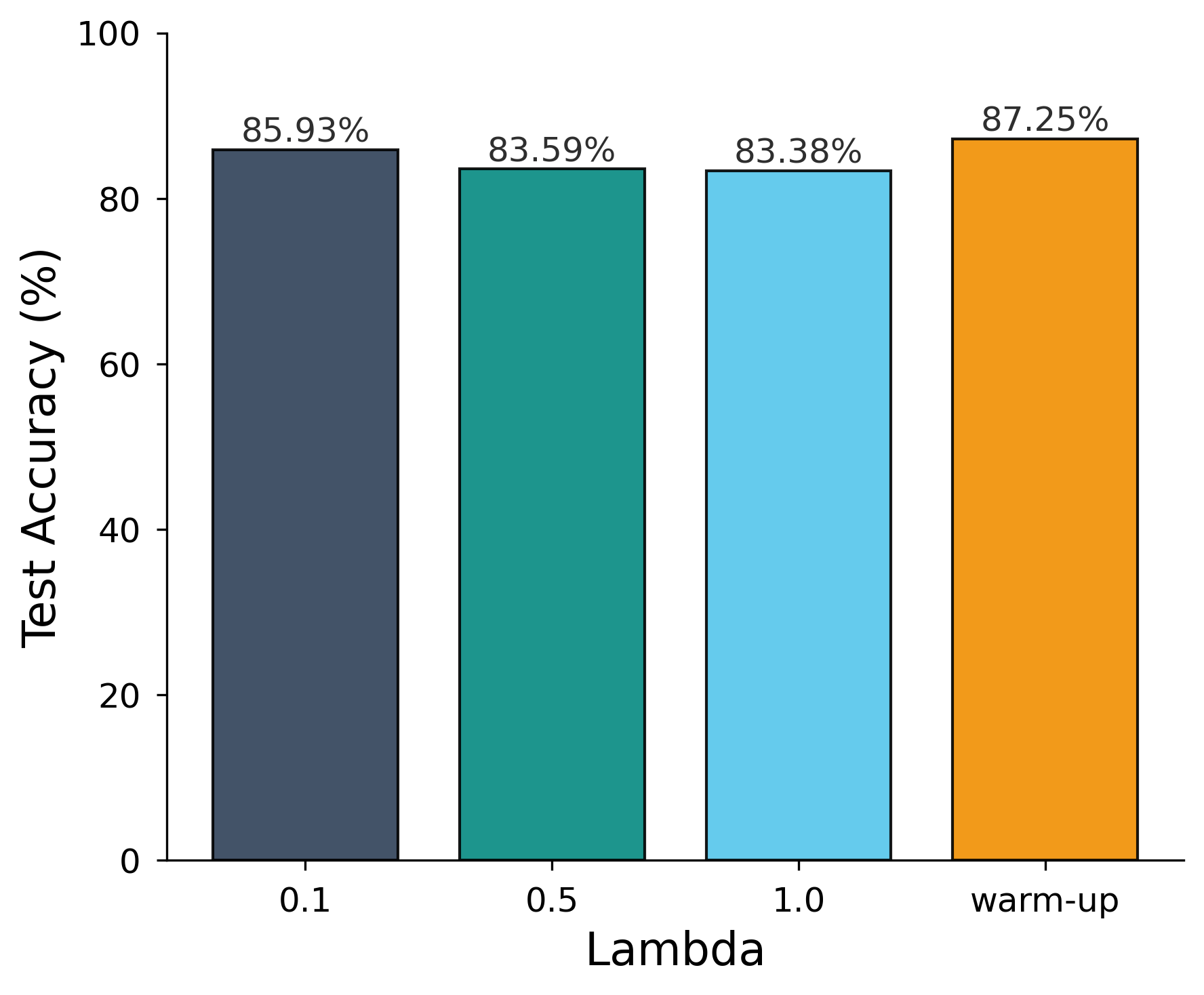}\label{fig:pic9b}}
    \caption{Impact of \textbf{(a)} different temperatures and \textbf{(b)} warm-up based $\lambda$ coefficient on FedAPA performance.}
    \label{fig:pic9}
\end{figure}We conduct experiments with different warm-up rounds, including 25, 50, and 100. The results are shown in Fig.~\ref{fig:pic8b}. A similar trend can be observed that the performance of FedAPA improves as the number of warm-up rounds increases from 25 to 50. However, when the number of warm-up rounds is further increased to 100, the performance slightly decreases. This suggests that a moderate number of warm-up rounds is beneficial for stabilizing training and guiding the model through a logical progression of learning stages. Too short warm-up rounds may cause the contrastive loss to dominate the training process early, leading to suboptimal performance. Conversely, too long warm-up rounds may delay the influence of contrastive loss, hindering the model's ability to learn robust representations. Therefore, we choose 50 as the default number of warm-up rounds for our experiments.

\subsubsection{Impact of different temperatures for contrastive loss terms and prototype aggregation} 
Temperature $\tau$ is a crucial hyperparameter in cosine similarity to adjust the tolerance for feature distribution~\cite{simclr}. We conduct experiments with different temperatures, including 0.1, 0.5, 0.7, and 1.5. The results are shown in Fig.~\ref{fig:pic9a}. It can be observed that the performance of FedAPA improves as the temperature increases from 0.1 to 0.5. However, when the temperature is further increased from 0.5 to 1.5, the performance slightly decreases. This suggests that a moderate temperature is beneficial for balancing the influence of negatives and preventing over-smoothing or over-sharpening of the similarity scores. Therefore, we choose 0.5 as the default temperature for our experiments.

\subsubsection{The impact of the coefficient ($\lambda$) on representation learning}
$\lambda$ is another importance hyperparameter that balances the contribution of contrastive representation learning and classification learning. We compare our warm-up based coefficient with three static coefficients, including 0.1, 0.5, and 1.0. The results are shown in Fig.~\ref{fig:pic9b}. As indicated by results, configurations with a static coefficient perform worse when compared to our warm-up based coefficient setting. The use of a static coefficient leads to the model encoder excessively focusing on either contrastive representation learning early (thereby compromising classification performance) or classification learning (thus hindering the development of globally informative representations). This suggests that by warming up the coefficient during a certain number of rounds, the model can gradually increase the emphasis on contrastive representation learning, leading to better overall performance. This adaptive approach results in enhanced performance of the model encoder and greater robustness. 

\subsubsection{Impact of similarity-aware prototype aggregation and other clients' prototypes involvement}
To validate the effectiveness of our proposed similarity-aware prototype aggregation method, we conduct experiments comparing it with a simple global average prototype aggregation method. \begin{table}[htbp]
\centering
\caption{The effect of similarity-aware prototype aggregation and other clients' prototypes involvement.}\label{tab:tab2}
\resizebox{\columnwidth}{!}{%
\begin{tabular}{c|ccc|ccc}
\toprule
    \textbf{Heterogeneous Setting} & \multicolumn{3}{c|}{\textbf{Statistical Data}} & \multicolumn{3}{c}{\textbf{Model Architecture}} \\\midrule
    \textbf{Metric} & \textbf{Acc.} & \textbf{F1} & \textbf{MAE} & \textbf{Acc.} & \textbf{F1} & \textbf{MAE} \\\midrule
    Global Avg. Proto & 74.02 & 73.59 & 0.71 & 54.75 & 47.89 & 1.62 \\\midrule
    Global Avg. Proto with Client Protos & \uline{82.03} & \uline{81.47} & \uline{0.38} & \uline{69.41} & \uline{67.20} & \uline{0.77} \\\midrule
    FedAPA & \textbf{87.25} & \textbf{85.91} & \textbf{0.23} & \textbf{80.31} & \textbf{78.94} & \textbf{0.48} \\
\bottomrule
\end{tabular}%
}
\end{table}
\begin{table}[htbp]
\centering
\caption{The theoretical and practical communication cost per communication round for each client using LargeConvNet4.}\label{tab:tab3}
\begin{tabular}{c|c|c}
\toprule 
  Methods & Theoretical & Practical \\\midrule 
  FedAvg\cite{fedavg} & $2 * \sum_w $ & \uline{3,710 KB} \\
  FedCaring\cite{caring} & $2 * \sum_w $ & \uline{3,710 KB} \\
  WiFederated\cite{wifederated} & $2 * \sum_w $ & \uline{3,710 KB} \\\midrule
    FedAPA & $ \sum_p * \Big(\sum_{N + 1} |\mathcal{C}|\Big)$ & \textbf{150.53 KB} \\
\bottomrule
\end{tabular}
\end{table}In the global average prototype aggregation method, the server computes the global prototype for each class by averaging the prototypes from all clients without considering their similarities. The results are shown in Table~\ref{tab:tab2}. It can be observed that our proposed similarity-aware prototype aggregation method significantly outperforms the simple global average prototype aggregation method in both data and model heterogeneity settings. This suggests that by considering the similarities between clients' prototypes, our method can effectively capture the relationships among clients and generate more informative global prototypes for each client. Additionally, we also evaluate the impact of involving other clients' prototypes in the local contrastive loss. The results indicate that incorporating other clients' prototypes further enhances the performance of our method, demonstrating the benefits of leveraging information from multiple clients to improve local representation learning.

\subsubsection{Communication overhead}
Apart from analyzing sensing performance, we also evaluate the communication cost per round for each client, which provides a comprehensive performance evaluation of our proposed framework. The communication overhead for each algorithm is assessed based on the calflops\cite{calflop} library. Table~\ref{tab:tab3} presents the detailed overhead results, where $\sum_w$ denotes the total number of parameters in a model, and $\sum_p$ indicates the total number of parameters in a prototype. It shows that our FedAPA approach has the least communication overhead when compared with those model-sharing based approaches. Fine-tuning approaches based on FedAvg, which involve uploading and downloading full models to and from the PS, have similar (and much larger) overhead. In contrast, FedAPA only requires the exchange of global and local prototypes, resulting in significantly lower communication costs. This makes FedAPA more suitable for deployment in resource-constrained environments where communication bandwidth is limited.

\section{Discussions}
Applying FL to cross-domain Wi-Fi sensing involves trade-offs among label availability, communication, computation, and scalability; no single design can optimize all dimensions at once. We next discuss key limitations of our approach.

\subsection{Scalability and Client Selection Strategy}
Partial client participation better reflects real deployments, where many devices cannot connect to the PS in every round \cite{clientselect}. Selecting only a subset of clients reduces uplink traffic but can degrade performance if the policy ignores representation diversity and system heterogeneity. While we do not evaluate FedAPA under partial participation, we suggest that a client selection strategy considering prototype diversity and device capability could enhance efficiency without compromising accuracy. In short, partial participation saves bandwidth, but selection must preserve informative clients while limiting overhead.

\subsection{Label Scarcity}
Our methods assume supervised learning. In practice, labeled CSI is scarce: ordinary users rarely annotate data and labels are expensive to curate in cross-domain Wi-Fi scenarios. This gap suggests moving from purely supervised FL toward semi-supervised or unsupervised formulations that exploit large volumes of unlabeled, crowdsourced CSI while preserving privacy.

\subsection{Computation Efficiency and Communication Cost}
Computation and communication jointly determine the practicality of FL systems. FedAPA communicates only global and local prototypes, which lowers communication overhead compared with approaches that exchange full model parameters. However, the additional prototype-based loss in the local training loop increases on-device computation, so there is a trade-off between reduced communication overhead and extra local compute.

\section{Conclusion}
This paper addresses statistical and model heterogeneity in distributed Wi-Fi CSI-based crowd counting with FedAPA, a federated learning method built around an adaptative prototype aggregation module. APA performs similarity-weighted aggregation of class prototypes, so clients with related distributions exchange compact class-level representations instead of full models, improving collaboration while preserving personalization and reducing communication overhead. Each local model is decomposed into an encoder and a classifier to enable shared representation learning with personalized decision layers. To align local and global knowledge, we adopt a hybrid local objective that couples classification loss with a prototype-contrastive loss under a warm-up schedule. We further provide a non-convex convergence analysis under standard smoothness and bounded-variance conditions to provide theoretical insight. Experiments on a Wi-Fi CSI crowd-counting dataset with six environments show consistent gains over FL-based Wi-Fi sensing baselines and substantially lower client-side bandwidth. This study focuses on a single sensing task; future work will extend FedAPA to multi-task settings for cross-task collaboration, move toward a Wi-Fi sensing foundation model with low communication cost, and investigate test-time environmental dynamics with unsupervised domain adaptation in the FL pipeline.
\normalem
\bibliographystyle{IEEEtran}

\appendices
\setcounter{assumption}{0}
\setcounter{theorem}{0}
\setcounter{lemma}{0}
\setcounter{corollary}{0}
\section{Convergence Analysis}
\label{app:convergence}

This appendix contains the full convergence analysis for FedAPA. It includes all notations, detailed assumptions, and complete proofs of lemmas and theorems in the main article.

\subsection{Notations}
Communication rounds are $t\in\{1,\dots,T\}$, local steps are $s\in\{0,\dots,S-1\}$, and $K:=TS$. Clients are $i\in\{1,\dots,N\}$. Client $i$ has an encoder $f_i(\cdot;w_{i,t}^\theta)$ and a classifier $w_{i,t}^h$ at round $t$; write $w_{i,t}=(w_{i,t}^\theta,w_{i,t}^h)$. The stochastic gradient at local step $(t,s)$ is $g_{i,t,s}$. Write $\mathbf Q_{i,t} := \{\mathbf q_{i,t}^c\}_{c\in\mathcal C}$ and $\mathbf Q_t := \mathbf Q_\tau(\mathbf P_t) := \{\mathbf Q_{i,t}\}_{i=1}^N$. Within round $t$, with $\mathbf P_t$ and $\mathbf Q_t$ fixed, client $i$ optimizes the local objective 
\begin{equation}
\label{eq:local-loss-causal}
\mathcal L_{i,t}(w)
:=
\mathcal L_{\mathrm{ce},i}(w)
+ \lambda_t\,\Phi_i(w; \mathbf Q_{t-1}, \mathbf P_{t-1}),
\qquad t \ge 1,
\end{equation}
where $\mathcal L_{\mathrm{ce}}$ is the cross-entropy loss and $\Phi_i$ collects the prototype-based contrastive terms $\mathcal L_{g} + \mathcal L_{c}$ defined in Section IV-C. For $t=1$, we assume an randomly initialized prototype set $(\mathbf P_0,\mathbf Q_0)$. At $t$-th round, clients update prototypes to $\mathbf P_t$ and the server constructs $\mathbf Q_t = \mathbf Q_\tau(\mathbf P_t)$, which are then used in the next round via \eqref{eq:local-loss-causal}. The schedule $\{\lambda_t\}$ follows the cosine warm-up in Section IV-C, and becomes constant after $T_{\mathrm{warm}}$ rounds. For convenience define the per-round stationarity metric 
\begin{equation} 
\label{eq:Git-def} G_{i,t}^2 := \sum_{s=0}^{S-1} \mathbb E\bigl[\|\nabla \mathcal L_{i,t}(w_{i,t,s})\|^2\bigr]. 
\end{equation}

\subsection{Assumptions}

We assume full participation. Let $L_{\inf}\in\mathbb R$ denote a uniform lower bound of all considered objectives. The following conditions are used in the analysis.

\begin{assumption}[Smoothness and lower boundedness]
\label{ass:A1-app}
The cross-entropy term $\mathcal L_{\mathrm{ce},i}$ is $L_{\mathrm{ce}}$-smooth,
and the regularizer $\Phi_i(\cdot;\mathbf Q,\mathbf P)$ is $L_\Phi$-smooth in $w$,
uniformly in $(\mathbf Q,\mathbf P)$. Fix round $t$ and the prototype stacks $\mathbf P_{t-1}$ and $\mathbf Q_\tau(\mathbf P_{t-1})$. Then the per-round objective $\mathcal L_{i,t}$ defined in \eqref{eq:local-loss-causal} is $L_t$-smooth:
\begin{equation}
\|\nabla\mathcal L_{i,t}(u) - \nabla\mathcal L_{i,t}(v)\|
\le L_t\|u-v\|,
\quad \forall u,v,
\end{equation}
with
$L_t = L_{\mathrm{ce}} + \lambda_t L_\Phi$, and after warm-up
$L_t \le L_{\max} := L_{\mathrm{ce}} + \lambda L_\Phi$. There also exists $L_{\inf}$ such that
\begin{equation}
\mathcal L_{i,t}(w) \ge L_{\inf}
\quad \forall i,t,w.
\end{equation}
\end{assumption}

\begin{assumption}[Stochastic gradient regularity]
\label{ass:A2-app}
At local step $(i,t,s)$, the stochastic gradient $g_{i,t,s}$ is an unbiased
estimate of the true gradient with bounded variance:
\begin{align}
\mathbb E[g_{i,t,s} \mid w_{i,t,s}]
&= \nabla \mathcal L_{i,t}(w_{i,t,s}), \\
\mathbb E\bigl[\|g_{i,t,s} - \nabla \mathcal L_{i,t}(w_{i,t,s})\|^2\bigr]
&\le \sigma^2,
\end{align}
for some $\sigma^2 \ge 0$, uniformly over all $i,t,s$. Moreover, there exists a global constant $G>0$ such that
\begin{equation}
\mathbb E\bigl[\|g_{i,t,s}\|^2\bigr] \le G^2
\qquad \text{for all } i,t,s.
\end{equation}
Under Assumption~\ref{ass:A3-app}, bounded encoder outputs and prototypes
(including padded ones) ensure that the gradients of the cross-entropy and
prototype-based contrastive losses remain uniformly bounded, so such $G$
exists.
\end{assumption}

\begin{assumption}[Bounded and Lipschitz representations]
\label{ass:A3-app}
The encoder output is uniformly bounded:
$\|\mathbf r_{w^\theta}(\mathbf h)\| \le 1$ for all $w^\theta,\mathbf h$.
Each prototype $\mathbf p_{i,t}^c$ is an empirical mean of such embeddings, so $\|\mathbf p_{i,t}^c\|\le 1$. The same bound holds for padded prototypes introduced in Section IV-B of main article, since they are averages of existing bounded prototypes. Moreover, the encoder is Lipschitz in its parameters: there exists
$L_{w^\theta}>0$ such that
\begin{equation}
\|\mathbf r_{w^\theta}(\mathbf h) - \mathbf r_{\hat w^\theta}(\mathbf h)\|
\le L_{w^\theta}\,\|w^\theta-\hat w^\theta\|,
\quad \forall w^\theta,\hat w^\theta,\mathbf h.
\end{equation}
\end{assumption}

\begin{assumption}[Lipschitz personalized aggregation]
\label{ass:A4-app}
Let $\mathbf P$ denote the stacked local prototypes and define
$\mathbf Q_\tau(\mathbf P)$ via the similarity-weighted aggregation rule
in Section IV-B. We assume that
$\mathbf Q_\tau(\cdot)$ is Lipschitz with respect to $\mathbf P$: there
exists a constant $L_{\mathrm{agg}}(\tau)>0$ such that
\begin{equation}
\|\mathbf Q_\tau(\mathbf P) - \mathbf Q_\tau(\mathbf P')\|_F
\le L_{\mathrm{agg}}(\tau)\,\|\mathbf P - \mathbf P'\|_F
\quad \forall\,\mathbf P,\mathbf P'.
\end{equation}
Under Assumption~\ref{ass:A3-app}, all prototypes (including padded ones) have
unit-norm bounds, and the softmax map with inverse temperature $1/\tau$
has Jacobian spectral norm $O(1/\tau)$. Hence $L_{\mathrm{agg}}(\tau)$ can be chosen on the order of $\sqrt{N}/\tau$. For example,
\begin{equation}
L_{\mathrm{agg}}(\tau) \;\le\; \frac{\sqrt{N}}{2\tau}.
\end{equation}
\end{assumption}

\begin{assumption}[Prototype-loss regularity]
\label{ass:A5-app}
For each client $i$, the prototype-based regularizer
$\Phi_i(w;\mathbf Q,\mathbf P)$ is Lipschitz in $(\mathbf Q,\mathbf P)$ and
uniformly bounded:

(A5.1) There exists $c_\Phi(\tau)>0$ such that
\begin{equation}
\begin{aligned}
|\Phi_i(w;\mathbf Q,\mathbf P)-\Phi_i(w;\mathbf Q',\mathbf P')|
&\le c_\Phi(\tau)\Bigl(\|\mathbf Q-\mathbf Q'\|_F \\
&\qquad + \|\mathbf P-\mathbf P'\|_F\Bigr),
\end{aligned}
\end{equation}
(A5.2) There exists $B_\Phi>0$ such that
$0\le \Phi_i(w;\mathbf Q,\mathbf P)\le B_\Phi$ for all $w,\mathbf Q,\mathbf P$.

These properties hold for both original and padded prototypes. Because the padded variants average unit-norm embeddings, they stay within the same bounded subset of the feature space.
\end{assumption}

\begin{assumption}[Warm-up schedule]
\label{ass:A6-app}
The warm-up weight $\{\lambda_t\}_{t\ge 1}$ is nondecreasing and there exist
$T_{\mathrm{warm}}<\infty$ and $\lambda\ge 0$ such that
\begin{equation}
\lambda_t = \lambda
\quad \text{for all } t\ge T_{\mathrm{warm}}.
\end{equation}
\end{assumption}

\subsection{Convergence results}
Based on the above assumptions, we have the following convergence result for FedAPA: 

We first state the standard descent result for one client with fixed prototypes.

\begin{lemma}[One-step descent with fixed prototypes]
\label{lem:one-step}
Let Assumptions~\ref{ass:A1-app} and \ref{ass:A2-app} hold. Fix round $t$ and client
$i$, and suppose $\mathbf P_t,\mathbf Q_t,\lambda_t$ are fixed during the
round. For step size $0<\eta\le 1/L_t$, the SGD update
$w_{i,t,s+1} = w_{i,t,s} - \eta g_{i,t,s}$ satisfies
\begin{equation}
\begin{aligned}
\mathbb E[\mathcal L_{i,t}(w_{i,t,s+1})]
&\le
\mathbb E[\mathcal L_{i,t}(w_{i,t,s})] \\
&- \eta\Big(1-\tfrac{L_t\eta}{2}\Big)
\mathbb E\big[\|\nabla\mathcal L_{i,t}(w_{i,t,s})\|^2\big]\\
&+ \tfrac{L_t\eta^2}{2}\sigma^2.
\label{eq:one-step}
\end{aligned}
\end{equation}
\end{lemma}

\begin{proof}
During round $t$, prototypes are fixed. by Assumption~1, $\mathcal{L}_{i, t}$ is $L_{t}$-smooth. Let the SGD update be $w_{i, t,s+1}=w_{i, t,s}-\eta\,g_{i, t,s}$. By the standard smoothness inequality,
\begin{equation}
\begin{aligned}
\label{eq:smoothness-per-client}
\mathcal{L}_{i, t}(w_{i, t,s+1})
&\le \mathcal{L}_{i, t}(w_{i, t,s})
- \eta\,\big\langle \nabla\mathcal{L}_{i, t}(w_{i, t,s}),\,g_{i, t,s}\big\rangle \\
&\quad
+ \frac{L_{i, t}\eta^2}{2}\,\|g_{i, t,s}\|^2.
\end{aligned}
\end{equation}
Taking conditional expectation given $w_{i, t,s}$ and using Assumption~3,
\begin{equation}
\mathbb{E}\!\left[\langle \nabla\mathcal{L}_{i, t}(w_{i, t,s}),\,g_{i, t,s}\rangle \mid w_{i, t,s}\right]
= \|\nabla\mathcal{L}_{i, t}(w_{i, t,s})\|^2,
\end{equation}
\begin{equation}
\mathbb{E}\!\left[\|g_{i, t,s}\|^2 \mid w_{i, t,s}\right]
\le \|\nabla\mathcal{L}_{i, t}(w_{i, t,s})\|^2 + \sigma^2.
\end{equation}
Plug these into \eqref{eq:smoothness-per-client} and then take total expectation to obtain
\begin{equation}
\begin{aligned}
\mathbb E[\mathcal L_{i,t}(w_{i,t,s+1})]
&\le
\mathbb E[\mathcal L_{i,t}(w_{i,t,s})] \\
&\quad 
- \eta\Big(1-\tfrac{L_t\eta}{2}\Big)
\mathbb E\big[\|\nabla\mathcal L_{i,t}(w_{i,t,s})\|^2\big]\\
&\quad
+ \tfrac{L_t\eta^2}{2}\sigma^2.
\end{aligned}
\end{equation}
which is \eqref{eq:one-step}.
\end{proof}

\begin{corollary}[$S$-step descent within a round]
\label{cor:within-round}
Under the conditions of Lemma~\ref{lem:one-step} and $\eta\le 1/L_t$,
\begin{equation}
\label{eq:within-round}
\mathbb E[\mathcal L_{i,t}(w_{i,t,S})]
\le
\mathbb E[\mathcal L_{i,t}(w_{i,t,0})]
- \tfrac{\eta}{2} G_{i,t}^2
+ \tfrac{L_t\eta^2}{2}S\sigma^2.
\end{equation}
\end{corollary}

We next control how much prototypes move between rounds and how this affects
the loss.

\begin{lemma}[Prototype movement]
\label{lem:proto-move}
Suppose Assumptions~\ref{ass:A2-app} and~\ref{ass:A3-app} hold.
Let $C_i := |\mathcal C_i|$ be the number of classes at client $i$ and
$C_{\max} := \max_{i} C_i$. Let
\begin{equation}
\|\Delta \mathbf P_t\|_F^2
=
\sum_{i=1}^N \|\mathbf P_{i,t}-\mathbf P_{i,t - 1}\|_F^2,
\end{equation} then

\textnormal{(a) Per-client movement.}
For any client $i$ and any round $t$,
\begin{equation}
\label{eq:P-move-client}
\mathbb{E}\big[\|\Delta \mathbf P_t\|_F\big]
\;\le\;
k_{\max}\,\eta\,S\,G,
\end{equation}
where
\begin{equation}
k_i := L_{w^\theta}\sqrt{C_i},
\qquad
k_{\max} := L_{w^\theta}\sqrt{C_{\max}},
\end{equation}
and hence $k_i \le k_{\max}$ for all $i$.

\textnormal{(b) Global movement.}
Consequently, the stacked prototype matrix
$\mathbf P_t := \{\mathbf P_{i,t}\}_{i=1}^N$ satisfies
\begin{equation}
\label{eq:P-move-global}
\mathbb{E}\big[\|\Delta \mathbf P_t\|_F\big]
\;\le\;
k_{\mathrm{glob}}\,\eta\,S\,G,
\end{equation}
where
\begin{equation}
k_{\mathrm{glob}}
:= L_{w^\theta}\sqrt{\sum_{i=1}^N C_i}
\;\le\;
L_{w^\theta}\sqrt{N C_{\max}}.
\end{equation}
\end{lemma}

\begin{proof}
Fix a client $i$ and a class $c\in\mathcal C_i$. By definition,
\begin{equation}
\begin{aligned}
\mathbf p_{i,t - 1}^c
&= \frac{1}{m_i^c}\sum_{\mathbf h\in\mathcal D_i^c}
\mathbf r_{w_{i,t - 1,S}^\theta}(\mathbf h),
\\
\mathbf p_{i,t}^c
&= \frac{1}{m_i^c}\sum_{\mathbf h\in\mathcal D_i^c}
\mathbf r_{w_{i,t,S}^\theta}(\mathbf h),
\end{aligned}
\end{equation}
By Assumption~\ref{ass:A3-app},
\begin{equation}
\big\|\mathbf r_{w_{i,t,S}^\theta}(\mathbf h) - \mathbf r_{w_{i,t - 1,S}^\theta}(\mathbf h)\big\|
\le L_{w^\theta}\,\|\Delta w_{i,t}^\theta\|
\quad \forall\,\mathbf h\in\mathcal D_i^c.
\end{equation}
Hence
\begin{equation}
\begin{aligned}
\|\mathbf p_{i,t}^c - \mathbf p_{i,t - 1}^c\|
&=
\Bigg\|
\frac{1}{m_i^c}\sum_{\mathbf h\in\mathcal D_i^c}
\big(
\mathbf r_{w_{i,t,S}^\theta}(\mathbf h)
-
\mathbf r_{w_{i,t - 1,S}^\theta}(\mathbf h)
\big)
\Bigg\|
\\
&\le
\frac{1}{m_i^c}\sum_{\mathbf h\in\mathcal D_i^c}
\big\|
\mathbf r_{w_{i,t,S}^\theta}(\mathbf h)
-
\mathbf r_{w_{i,t - 1,S}^\theta}(\mathbf h)
\big\|
\\
&\le
\frac{1}{m_i^c}\sum_{\mathbf h\in\mathcal D_i^c}
L_{w^\theta}\,\|\Delta w_{i,t}^\theta\|
\\
&=
L_{w^\theta}\,\|\Delta w_{i,t}^\theta\|.
\end{aligned}
\end{equation}
Therefore,
\begin{equation}
\label{eq:proto-diff-sq}
\|\mathbf p_{i,t}^c - \mathbf p_{i,t - 1}^c\|^2
\le
L_{w^\theta}^2\,\|\Delta w_{i,t}^\theta\|^2.
\end{equation}

The encoder parameters follow the SGD recursion
\begin{equation}
w_{i,t,s+1}^\theta = w_{i,t,s}^\theta - \eta\,g_{i,t,s}^\theta,
\end{equation}
where $s=0,\dots,S-1$, so
\begin{equation}
\Delta w_{i,t}^\theta
= w_{i,t,S}^\theta - w_{i,t,0}^\theta
= -\eta\sum_{s=0}^{S-1} g_{i,t,s}^\theta.
\end{equation}
Then
\begin{equation}
\label{eq:encoder-diff-sq}
\begin{aligned}
\|\Delta w_{i,t}^\theta\|^2
&=
\eta^2\Big\|\sum_{s=0}^{S-1} g_{i,t,s}^\theta\Big\|^2
\\
&\le
\eta^2 S \sum_{s=0}^{S-1} \|g_{i,t,s}^\theta\|^2
\\
&\le
\eta^2 S \sum_{s=0}^{S-1} \|g_{i,t,s}\|^2,
\end{aligned}
\end{equation}
where we used $\|g_{i,t,s}^\theta\|\le \|g_{i,t,s}\|$ and the Cauchy-Schwarz inequality. Taking expectations and using Assumption~\ref{ass:A2-app}
($\mathbb E\|g_{i,t,s}\|^2\le G^2$) yields
\begin{equation}
\label{eq:encoder-diff-second-moment}
\mathbb E\bigl[\|\Delta w_{i,t}^\theta\|^2\bigr]
\le
\eta^2 S \sum_{s=0}^{S-1} \mathbb E\bigl[\|g_{i,t,s}\|^2\bigr]
\le
\eta^2 S^2 G^2.
\end{equation}
Combining \eqref{eq:proto-diff-sq} and \eqref{eq:encoder-diff-second-moment},
we obtain
\begin{equation}
\mathbb E\bigl[\|\mathbf p_{i,t}^c - \mathbf p_{i,t - 1}^c\|^2\bigr]
\le
L_{w^\theta}^2 \eta^2 S^2 G^2.
\end{equation}

Part (a): per-client bound.
For a fixed client $i$, the Frobenius norm of its prototype matrix satisfies
\begin{equation}
\|\mathbf P_{i,t}-\mathbf P_{i,t - 1}\|_F^2
=
\sum_{c\in\mathcal C_i} \|\mathbf p_{i,t}^c - \mathbf p_{i,t - 1}^c\|^2.
\end{equation}
Taking expectations and applying the bound above,
\begin{equation}
\begin{aligned}
\mathbb E\bigl[\|\mathbf P_{i,t}-\mathbf P_{i,t - 1}\|_F^2\bigr]
&\le
\sum_{c\in\mathcal C_i}
L_{w^\theta}^2 \eta^2 S^2 G^2
\\
&=
\mathcal{C}_i\,L_{w^\theta}^2 \eta^2 S^2 G^2.
\end{aligned}
\end{equation}
Using Jensen's inequality, we obtain
\begin{align}
\mathbb E \bigl[\|\mathbf P_{i,t}-\mathbf P_{i,t - 1}\|_F \bigr]
&\le
\sqrt{\mathbb E\bigl[\|\mathbf P_{i,t}-\mathbf P_{i,t - 1}\|_F^2\bigr]}
\\
&\le
L_{w^\theta}\eta S G \sqrt{C_i}
=: k_i\,\eta\,S\,G.
\nonumber
\end{align}
which yields \eqref{eq:P-move-client}.
Since $C_i\le C_{\max}$ for every $i$, set $k_{\max} := L_{w^\theta}\sqrt{C_{\max}}$; then $k_i\le k_{\max}$.

Part (b): global bound.
The stacked prototype matrix is
$\mathbf P_t := \{\mathbf P_{i,t}\}_{i=1}^N$. 
Taking expectations and using the per-client bound,
\begin{equation}
\begin{aligned}
\mathbb E\bigl[\|\Delta \mathbf P_t\|_F^2\bigr]
&\le
\sum_{i=1}^N C_i\,L_{w^\theta}^2 \eta^2 S^2 G^2
\\
&=
L_{w^\theta}^2 \eta^2 S^2 G^2 \sum_{i=1}^N C_i.
\end{aligned}
\end{equation}
Again by Jensen inequality,
\begin{equation}
\begin{aligned}
\mathbb E\bigl[\|\Delta \mathbf P_t\|_F\bigr]
&\le
\sqrt{\mathbb E\bigl[\|\Delta \mathbf P_t\|_F^2\bigr]}
\\
&\le
L_{w^\theta}\eta S G \sqrt{\sum_{i=1}^N C_i}
=: k_{\mathrm{glob}}\,\eta\,S\,G.
\end{aligned}
\end{equation}
Finally, since $C_i\le C_{\max}$ for all $i$, we have
\begin{equation}
\sum_{i=1}^N C_i \le \sum_{i=1}^N C_{\max} = N C_{\max},
\end{equation}
so $k_{\mathrm{glob}}\le L_{w^\theta}\sqrt{N C_{\max}}$.
This proves part (b) and completes the proof.
\end{proof}

\begin{lemma}[Loss change due to prototype refresh]
\label{lem:refresh-correct}
Define
\begin{equation}
\label{eq:Li-t-def}
\mathcal L_{i,t}(w)
:= \mathcal L_{\mathrm{ce},i}(w)
+ \lambda_t\,\Phi_i(w;\mathbf Q_{t-1},\mathbf P_{t-1}),
\end{equation}
\begin{equation}
\label{eq:Li-t1-def}
\mathcal L_{i,t+1}(w)
:= \mathcal L_{\mathrm{ce},i}(w)
+ \lambda_{t+1}\,\Phi_i(w;\mathbf Q_{t},\mathbf P_{t}).
\end{equation}
Under Assumptions~\ref{ass:A4-app} and~\ref{ass:A5-app}, for any client $i$,
\begin{equation}
\label{eq:refresh-bound-correct}
\begin{aligned}
\mathbb E\Big[
\big|\mathcal L_{i,t+1}(w_{i,t,S})
- \mathcal L_{i,t}(w_{i,t,S})\big|
\Big]
&\le
|\lambda_{t+1}-\lambda_t|\,B_\Phi \\
&+ \lambda_{t+1}\,\Gamma(\tau)\,k_{\mathrm{glob}}\,\eta\,S\,G,
\end{aligned}
\end{equation}
where $\Gamma(\tau) := c_\Phi(\tau)\bigl(1+L_{\mathrm{agg}}(\tau)\bigr)$ and
$k_{\mathrm{glob}}$ is as in Lemma~\ref{lem:proto-move}(b).
\end{lemma}

\begin{proof}
Fix a client $i$ and a round $t\ge 1$.
We want to control the difference between the objectives at rounds $t$ and $t+1$
when both are evaluated at the same parameter vector
$w_{i,t,S}$, namely
\(
\mathcal L_{i,t+1}(w_{i,t,S}) - \mathcal L_{i,t}(w_{i,t,S}).
\)

\paragraph{Step 1: Separate schedule and prototype effects.}
By definitions \eqref{eq:Li-t-def} and \eqref{eq:Li-t1-def},
for any $w$ we have. For conciseness, introduce the shorthand
\begin{equation}
\Phi_{i,\tau}(w) := \Phi_i\bigl(w;\mathbf Q_{\tau},\mathbf P_{\tau}\bigr)
\end{equation}
for any round index $\tau$; we will only need $\Phi_{i,t}$ and $\Phi_{i,t-1}$ below.
\begin{equation}
\begin{aligned}
\mathcal L_{i,t+1}(w) - \mathcal L_{i,t}(w)
&=
\mathcal L_{\mathrm{ce},i}(w)
+ \lambda_{t+1}\Phi_{i,t}(w)
\nonumber\\[-1mm]
&\quad
- \mathcal L_{\mathrm{ce},i}(w)
- \lambda_t\Phi_{i,t-1}(w) \\
&=
\lambda_{t+1}\Phi_{i,t}(w)
- \lambda_t\Phi_{i,t-1}(w).
\label{eq:Li-diff-start}
\end{aligned}
\end{equation}
Add and subtract the term
$\lambda_{t+1}\Phi_i(w;\mathbf Q_{t-1},\mathbf P_{t-1})$ to get
\begin{align}
\mathcal L_{i,t+1}(w) - \mathcal L_{i,t}(w)
&=
\bigl(\lambda_{t+1}-\lambda_t\bigr)
\Phi_{i,t-1}(w)
\nonumber\\[-1mm]
&\quad
+ \lambda_{t+1}\Bigl(
\Phi_{i,t}(w) - \Phi_{i,t-1}(w)
\Bigr).
\label{eq:Li-diff-decomp}
\end{align}
We now bound these two parts separately.

\paragraph{Step 2: Bound the schedule-change term.}
Take absolute values in \eqref{eq:Li-diff-decomp}:
\begin{align}
\bigl|\mathcal L_{i,t+1}(w) - \mathcal L_{i,t}(w)\bigr|
&\le
\bigl|\lambda_{t+1}-\lambda_t\bigr|\,
\bigl|\Phi_{i,t-1}(w)\bigr|
\nonumber\\
&\quad
+ \lambda_{t+1}
\bigl|\Phi_{i,t}(w) - \Phi_{i,t-1}(w)\bigr|.
\label{eq:Li-diff-abs}
\end{align}
By Assumption~\ref{ass:A5-app}(A5.2),
the prototype-based regularizer is uniformly bounded:
\begin{equation}
0 \le \Phi_i(w;\mathbf Q,\mathbf P) \le B_\Phi
\qquad\text{for all }w,\mathbf Q,\mathbf P.
\end{equation}
Hence, for the first term in \eqref{eq:Li-diff-abs} we have
\begin{equation}
\bigl|\lambda_{t+1}-\lambda_t\bigr|\,
\bigl|\Phi_{i,t-1}(w)\bigr|
\le
|\lambda_{t+1}-\lambda_t|\,B_\Phi.
\label{eq:sched-bound}
\end{equation}

\paragraph{Step 3: Bound the prototype-refresh term via Lipschitz properties.}
For the second term in \eqref{eq:Li-diff-abs}, apply
Assumption~\ref{ass:A5-app}(A5.1), which states that
$\Phi_i$ is Lipschitz in $(\mathbf Q,\mathbf P)$.
Define the shorthand differences
\begin{equation}
\Delta \mathbf Q_t := \mathbf Q_t-\mathbf Q_{t-1},
\qquad
\Delta \mathbf P_t := \mathbf P_t-\mathbf P_{t-1}.
\end{equation}
Then
\begin{equation}
\begin{aligned}
\bigl|\Phi_i(w;\mathbf Q_t,\mathbf P_t)
&- \Phi_i(w;\mathbf Q_{t-1},\mathbf P_{t-1})\bigr|
\\
&\le
c_\Phi(\tau)\Bigl(
\|\Delta \mathbf Q_t\|_F
\\
&\qquad
+ \|\Delta \mathbf P_t\|_F
\Bigr).
\label{eq:Phi-Lip}
\end{aligned}
\end{equation}
We now relate $\|\mathbf Q_t-\mathbf Q_{t-1}\|_F$ to
$\|\mathbf P_t-\mathbf P_{t-1}\|_F$ by using
Assumption~\ref{ass:A4-app} (Lipschitz personalized aggregation):
\begin{equation}
\|\Delta \mathbf Q_t\|_F
=
\|\mathbf Q_\tau(\mathbf P_t) - \mathbf Q_\tau(\mathbf P_{t-1})\|_F
\le
L_{\mathrm{agg}}(\tau)\,\|\Delta \mathbf P_t\|_F.
\label{eq:Q-Lip}
\end{equation}
Substituting \eqref{eq:Q-Lip} into \eqref{eq:Phi-Lip} gives
\begin{equation}
\begin{aligned}
\bigl|\Phi_i(w;\mathbf Q_t,\mathbf P_t)
&- \Phi_i(w;\mathbf Q_{t-1},\mathbf P_{t-1})\bigr|
\\
&\le
c_\Phi(\tau)\Bigl(
L_{\mathrm{agg}}(\tau)\|\Delta \mathbf P_t\|_F
\\
&\qquad
+ \|\Delta \mathbf P_t\|_F
\Bigr)
\nonumber\\
&=
c_\Phi(\tau)\bigl(1+L_{\mathrm{agg}}(\tau)\bigr)
\|\Delta \mathbf P_t\|_F.
\end{aligned}
\end{equation}
Define
\begin{equation}
\Gamma(\tau)
:=
c_\Phi(\tau)\bigl(1+L_{\mathrm{agg}}(\tau)\bigr),
\end{equation}
so that
\begin{equation}
\begin{aligned}
\bigl|\Phi_i(w;\mathbf Q_t,\mathbf P_t)
&- \Phi_i(w;\mathbf Q_{t-1},\mathbf P_{t-1})\bigr|
\\
&\le
\Gamma(\tau)\,\|\Delta \mathbf P_t\|_F.
\end{aligned}
\label{eq:Phi-diff-bound}
\end{equation}

Multiplying \eqref{eq:Phi-diff-bound} by $\lambda_{t+1}$ yields
\begin{equation}
\begin{aligned}
\lambda_{t+1}
\bigl|\Phi_i(w;\mathbf Q_t,\mathbf P_t)
&- \Phi_i(w;\mathbf Q_{t-1},\mathbf P_{t-1})\bigr|
\\
&\le
\lambda_{t+1}\,\Gamma(\tau)\,\|\Delta \mathbf P_t\|_F.
\end{aligned}
\label{eq:proto-term-bound}
\end{equation}

\paragraph{Step 4: Combine bounds and evaluate at $w_{i,t,S}$.}
Define the per-round loss change
\begin{equation}
\Delta \mathcal L_{i,t}(w)
:=
\mathcal L_{i,t+1}(w) - \mathcal L_{i,t}(w).
\end{equation}
Substitute \eqref{eq:sched-bound} and \eqref{eq:proto-term-bound}
into \eqref{eq:Li-diff-abs}:
\begin{equation}
\begin{aligned}
\bigl|\Delta \mathcal L_{i,t}(w)\bigr|
&\le
|\lambda_{t+1}-\lambda_t|\,B_\Phi \\
&\quad
+ \lambda_{t+1}\,\Gamma(\tau)\,
\|\Delta \mathbf P_t\|_F.
\end{aligned}
\label{eq:Li-diff-general}
\end{equation}
Now set $w = w_{i,t,S}$, the parameters after the $S$ local steps
in round $t$:
\begin{equation}
\begin{aligned}
\bigl|\Delta \mathcal L_{i,t}(w_{i,t,S})\bigr|
&\le
|\lambda_{t+1}-\lambda_t|\,B_\Phi \\
&\quad
+ \lambda_{t+1}\,\Gamma(\tau)\,\|\Delta \mathbf P_t\|_F.
\end{aligned}
\label{eq:Li-diff-wits}
\end{equation}

\paragraph{Step 5: Take expectations and use the prototype-movement bound.}
Finally, take expectations on both sides of \eqref{eq:Li-diff-wits}:
\begin{equation}
\begin{aligned}
\mathbb E\Big[
\bigl|\Delta \mathcal L_{i,t}(w_{i,t,S})\bigr|
\Big]
&\le
|\lambda_{t+1}-\lambda_t|\,B_\Phi \\
&\quad + \lambda_{t+1}\,\Gamma(\tau)\,
\mathbb E\bigl[\|\Delta \mathbf P_t\|_F\bigr].
\end{aligned}
\label{eq:Li-diff-exp}
\end{equation}
By Lemma~\ref{lem:proto-move}(b) in its causal form,
the stacked prototype movement is bounded by
\begin{equation}
\mathbb E\bigl[\|\Delta \mathbf P_t\|_F\bigr]
\le
k_{\mathrm{glob}}\,\eta\,S\,G.
\label{eq:proto-move-simple}
\end{equation}
Substituting \eqref{eq:proto-move-simple} into \eqref{eq:Li-diff-exp} gives
\begin{equation}
\begin{aligned}
\mathbb E\Big[
\bigl|\Delta \mathcal L_{i,t}(w_{i,t,S})\bigr|
\Big]
&\le
|\lambda_{t+1}-\lambda_t|\,B_\Phi \\
&\quad
+ \lambda_{t+1}\,\Gamma(\tau)\,
k_{\mathrm{glob}}\,\eta\,S\,G,
\end{aligned}
\end{equation}
which is exactly \eqref{eq:refresh-bound-correct}. This completes the proof.
\end{proof}

We now connect round $t$ and round $t+1$ for client $i$.

\begin{theorem}[One-round deviation for an arbitrary client]
\label{thm:one-round-causal}
Let Assumptions~\ref{ass:A1-app} to \ref{ass:A5-app} hold.
For any client $i$ and any $t\ge 1$, with stepsize
$0<\eta\le 1/L_t$, we have
\begin{align}
\mathbb E\big[\mathcal L_{i,t+1}(w_{i,t+1,0})\big]
&\le
\mathbb E\big[\mathcal L_{i,t}(w_{i,t,0})\big]
- \tfrac{\eta}{2} G_{i,t}^2
\nonumber\\
&\quad
+ \tfrac{L_t\eta^2}{2}S\sigma^2
+ \lambda_{t+1}\,\Gamma(\tau)\,k_{\mathrm{glob}}\,\eta\,S\,G
\nonumber\\
&\quad
+ |\lambda_{t+1}-\lambda_t|\,B_\Phi.
\label{eq:one-round-dev-causal}
\end{align}
\end{theorem}

\begin{proof}
By Corollary~\ref{cor:within-round} and the bound $\eta\le 1/L_t$,
we have $1-\tfrac{L_t\eta}{2}\ge\tfrac12$, so
\begin{equation}
\label{eq:within-round-again}
\mathbb E[\mathcal L_{i,t}(w_{i,t,S})]
\le
\mathbb E[\mathcal L_{i,t}(w_{i,t,0})]
- \tfrac{\eta}{2} G_{i,t}^2
+ \tfrac{L_t\eta^2}{2}S\sigma^2.
\end{equation}
By definition $w_{i,t+1,0} := w_{i,t,S}$, and
\begin{equation}
\mathbb E[\mathcal L_{i,t+1}(w_{i,t+1,0})]
=
\mathbb E[\mathcal L_{i,t+1}(w_{i,t,S})].
\end{equation}
Add and subtract $\mathcal L_{i,t}(w_{i,t,S})$ inside the expectation:
\begin{equation}
\label{eq:add-sub-Lt}
\begin{aligned}
\mathbb E[\mathcal L_{i,t+1}(w_{i,t,S})]
&=
\mathbb E[\mathcal L_{i,t}(w_{i,t,S})] \\
&\quad+
\mathbb E\bigl[
\mathcal L_{i,t+1}(w_{i,t,S}) - \mathcal L_{i,t}(w_{i,t,S})
\bigr].
\end{aligned}
\end{equation}
Apply Lemma~\ref{lem:refresh-correct}:
\begin{equation}
\label{eq:refresh-apply}
\begin{aligned}
\mathbb E\bigl[
\mathcal L_{i,t+1}(w_{i,t,S}) - \mathcal L_{i,t}(w_{i,t,S})
\bigr]
&\le
|\lambda_{t+1}-\lambda_t|\,B_\Phi \\
&\quad + \lambda_{t+1}\,\Gamma(\tau)\,k_{\mathrm{glob}}\,\eta\,S\,G.
\end{aligned}
\end{equation}
Combining \eqref{eq:add-sub-Lt} and \eqref{eq:refresh-apply}, we obtain
\begin{equation}
\label{eq:Ltp1-vs-LtS}
\begin{aligned}
\mathbb E[\mathcal L_{i,t+1}(w_{i,t+1,0})]
&=
\mathbb E[\mathcal L_{i,t+1}(w_{i,t,S})] \\
&\le
\mathbb E[\mathcal L_{i,t}(w_{i,t,S})] \\
&\quad+ |\lambda_{t+1}-\lambda_t|\,B_\Phi \\
&\quad+ \lambda_{t+1}\,\Gamma(\tau)\,k_{\mathrm{glob}}\,\eta\,S\,G.
\end{aligned}
\end{equation}
Finally, plug \eqref{eq:within-round-again} into \eqref{eq:Ltp1-vs-LtS}
to eliminate $\mathcal L_{i,t}(w_{i,t,S})$:
\begin{align}
\mathbb E[\mathcal L_{i,t+1}(w_{i,t+1,0})]
&\le
\mathbb E[\mathcal L_{i,t}(w_{i,t,0})]
- \tfrac{\eta}{2} G_{i,t}^2
+ \tfrac{L_t\eta^2}{2}S\sigma^2
\nonumber\\
&\quad
+ |\lambda_{t+1}-\lambda_t|\,B_\Phi
\nonumber\\
&\quad
+ \lambda_{t+1}\,\Gamma(\tau)\,k_{\mathrm{glob}}\,\eta\,S\,G,
\end{align}
which is exactly \eqref{eq:one-round-dev-causal}. This completes the proof.
\end{proof}

Building on Theorem~\ref{thm:one-round-causal}, we have the following overall convergence guarantee after the warm-up phase.

\begin{theorem}[$\varepsilon$-stationarity after warm-up]
\label{thm:epsilon-stationarity-app}
Suppose Assumptions~\ref{ass:A1-app} to \ref{ass:A6-app} hold and that
$\lambda_t \equiv \lambda$ and $L_t \le L_{\max}$ for all $t \ge T_{\mathrm{warm}}$. We consider $T$ communication rounds after $T_{\mathrm{warm}}$; for notational convenience, we re-index these rounds as $t \in {1,\dots,T}$. Each round has $S$ local steps, so there are $K := T S$ total local updates in this phase. For any client $i$, the iterates of FedAPA satisfy
\begin{equation}
\begin{aligned}
\frac{1}{K}\sum_{t=1}^{T} G_{i,t}^2
&\le
\frac{2\Delta_i}{K\eta}
\\
&\quad
+ L_{\max}\eta\sigma^2
+ 2\lambda\,\Gamma(\tau)\,k_{\mathrm{glob}}\,G,
\label{eq:avg-grad-bound-app}
\end{aligned}
\end{equation}
for any stepsize $0<\eta\le 1/L_{\max}$, where
\begin{equation}
\Delta_i := \mathbb E\big[\mathcal L_{i,T_{\mathrm{warm}}}(w_{i,T_{\mathrm{warm}},0})\big]
- L_{\inf},
\end{equation}
\begin{equation}
\Gamma(\tau) := c_\Phi(\tau)\bigl(1 + L_{\mathrm{agg}}(\tau)\bigr),
\end{equation}
and $k_{\mathrm{glob}}$ is the prototype-movement constant from Lemma~\ref{lem:proto-move}(b).

Moreover, for any target $\varepsilon>0$, if the parameters satisfy
\begin{align}
0 < \eta
&\le
\min\Big\{\frac{1}{L_{\max}},\; \frac{\varepsilon}{3L_{\max}\sigma^2}\Big\},
\label{eq:cond-eta-app}\\[4pt]
\lambda
&\le
\frac{\varepsilon}{6\,\Gamma(\tau)\,k_{\mathrm{glob}}\,G},
\label{eq:cond-lambda-app}\\[4pt]
T
&\ge
\frac{6\Delta_i}{\varepsilon\,\eta\,S},
\label{eq:cond-T-app}
\end{align}

then FedAPA achieves $\varepsilon$-stationarity for client $i$:
\begin{equation}
\frac{1}{K}\sum_{t=1}^{T} G_{i,t}^2
\;\le\; \varepsilon.
\end{equation}
\end{theorem}
The bound in \eqref{eq:avg-grad-bound-app} shows that, for each client $i$, the last term $2\lambda\,\Gamma(\tau)\,k_{\mathrm{glob}}\,G$ forms an asymptotic error floor due to prototype refresh: for fixed $\lambda>0$, the first two terms can be made arbitrarily small by taking $K$ large and $\eta$ small, but the floor remains. Our $\varepsilon$-stationarity result is obtained by choosing $\eta$, $T$, and $\lambda$ so that each term is at most $\varepsilon/3$, as in \eqref{eq:cond-eta-app} to \eqref{eq:cond-T-app}; in particular, this requires $\lambda = O\!\bigl(\varepsilon/(\Gamma(\tau)k_{\mathrm{glob}}G)\bigr)$. The dependence on similarity-weighted prototypes appears only through the factor $\lambda\,\Gamma(\tau)\,k_{\mathrm{glob}}$. Our analysis assumes a finite constant $\Gamma(\tau)$ for each $\tau$. Intuitively, larger $\tau$ makes the softmax weights flatter and the aggregation less sensitive to prototype differences, which can reduce $\Gamma(\tau)$ in practice, although such monotonic behavior is not required or proved by the theorem.

\begin{proof}
Fix a client $i$ and consider the $T$ rounds after $T_{\mathrm{warm}}$, where
$\lambda_t \equiv \lambda$ and $L_t \le L_{\max}$.
For notational simplicity, we re-index these rounds as $t=1,\dots,T$.
We write
\begin{equation}
\mathcal L_{i,t}
:= \mathcal L_{i,t}(w_{i,t,0}).
\end{equation}

Step 1: One-round inequality after warm-up.

From Theorem~\ref{thm:one-round-causal}, using
$\lambda_{t+1} = \lambda_t = \lambda$ for all $t \ge T_{\mathrm{warm}}$ so that
the schedule-change term vanishes, and using $L_t \le L_{\max}$, we obtain for
each $t \ge T_{\mathrm{warm}}$:
\begin{equation}
\begin{aligned}
\mathbb E[\mathcal L_{i,t+1}]
&\le
\mathbb E[\mathcal L_{i,t}]
- \tfrac{\eta}{2}G_{i,t}^2
+ \frac{L_{\max}\eta^2}{2}S\sigma^2 \\
&\quad
+ \lambda\,\Gamma(\tau)\,k_{\mathrm{glob}}\,\eta\,S\,G.
\end{aligned}
\label{eq:one-round-warm-correct}
\end{equation}

Step 2: Telescoping over $T$ rounds.

Sum \eqref{eq:one-round-warm-correct} from
$t = T_{\mathrm{warm}}$ to $T_{\mathrm{warm}}+T-1$:
\begin{equation}
\begin{aligned}
\mathbb E[\mathcal L_{i,T_{\mathrm{warm}}+T}]
&\le
\mathbb E[\mathcal L_{i,T_{\mathrm{warm}}}]
- \frac{\eta}{2} \sum_{t} G_{i,t}^2
+ \frac{L_{\max}\eta^2}{2}TS\sigma^2 \\
&\quad
+ T\,\lambda\,\Gamma(\tau)\,k_{\mathrm{glob}}\,\eta\,S\,G.
\end{aligned}
\end{equation}
Rearrange:
\begin{equation}
\begin{aligned}
\frac{\eta}{2} \sum_{t} G_{i,t}^2
&\le
\mathbb E[\mathcal L_{i,T_{\mathrm{warm}}}]
- \mathbb E[\mathcal L_{i,T_{\mathrm{warm}}+T}]
\\
&\quad
+ \frac{L_{\max}\eta^2}{2}TS\sigma^2
+ T\,\lambda\,\Gamma(\tau)\,k_{\mathrm{glob}}\,\eta\,S\,G.
\end{aligned}
\label{eq:sum-G2-correct}
\end{equation}
By Assumption~\ref{ass:A1-app}, $\mathcal L_{i,t}(w)\ge L_{\inf}$, so
\(\mathbb E[\mathcal L_{i,T_{\mathrm{warm}}+T}] \ge L_{\inf}\).
Define
\begin{equation}
\Delta_i
:= \mathbb E[\mathcal L_{i,T_{\mathrm{warm}}}] - L_{\inf} \;\ge\; 0.
\end{equation}
Then \eqref{eq:sum-G2-correct} yields
\begin{equation}
\frac{\eta}{2} \sum_{t} G_{i,t}^2
\le
\Delta_i
+ \frac{L_{\max}\eta^2}{2}TS\sigma^2
+ T\,\lambda\,\Gamma(\tau)\,k_{\mathrm{glob}}\,\eta\,S\,G.
\end{equation}

Step 3: Normalize by $K=TS$.

Recall that
\begin{equation}
\sum_{t} G_{i,t}^2
=
\sum_{t=1}^{T}\sum_{s=0}^{S-1}
\mathbb E\big[\|\nabla\mathcal L_{i,t}(w_{i,t,s})\|^2\big],
\end{equation}
and $K := TS$. Divide by $K$:
\begin{equation}
\begin{aligned}
\frac{\eta}{2}\cdot
\frac{1}{K}\sum_{t=1}^{T}\sum_{s=0}^{S-1}
\mathbb E\big[\|\nabla\mathcal L_{i,t}(w_{i,t,s})\|^2\big]
&\le
\frac{\Delta_i}{K}
+ \frac{L_{\max}\eta^2}{2}\sigma^2\\
&\quad
+ \lambda\,\Gamma(\tau)\,k_{\mathrm{glob}}\,\eta\,G.
\end{aligned}
\end{equation}
Multiply both sides by $2/\eta$ to obtain \eqref{eq:avg-grad-bound-app}.

Step 4: Enforce the $\varepsilon$-conditions.

Let $\varepsilon>0$ be given.

(i) The step size constraint \eqref{eq:cond-eta-app} ensures
\begin{equation}
L_{\max}\eta\sigma^2 \le \frac{\varepsilon}{3}.
\end{equation}

(ii) The prototype weight condition \eqref{eq:cond-lambda-app} ensures
\begin{equation}
2\lambda\,\Gamma(\tau)\,k_{\mathrm{glob}}\,G \le \frac{\varepsilon}{3}.
\end{equation}

(iii) If we take \eqref{eq:cond-T-app} so that
$K = TS \ge 6\Delta_i / (\varepsilon\eta)$, then
\begin{equation}
\frac{2\Delta_i}{K\eta} \le \frac{\varepsilon}{3}.
\end{equation}

Substituting these three bounds into \eqref{eq:avg-grad-bound-app} yields
\begin{equation}
\frac{1}{K}\sum_{t=1}^{T}\sum_{s=0}^{S-1}
\mathbb E\big[\|\nabla\mathcal L_{i,t}(w_{i,t,s})\|^2\big]
\;\le\;
\frac{\varepsilon}{3}
+ \frac{\varepsilon}{3}
+ \frac{\varepsilon}{3}
= \varepsilon,
\end{equation}
which completes the proof.
\end{proof}

\end{document}